\documentclass[opre,nonblindrev]{style/informs}

\OneAndAHalfSpacedXI
\TheoremsNumberedThrough     
\EquationsNumberedThrough    
\ECRepeatTheorems

\usepackage{natbib}
 \bibpunct[, ]{(}{)}{,}{a}{}{,}%
 \def\bibfont{\small}%

\usepackage{hyperref}       
\hypersetup{
  colorlinks=true,
  linkcolor=red,
  urlcolor=magenta,
  citecolor=blue
}

\usepackage{amsmath,amssymb,amsfonts}
\usepackage{graphicx,color}
\usepackage{url}
\usepackage{subfiles}
\usepackage{booktabs}
\usepackage{mathtools}
\usepackage{microtype}
\usepackage[english]{babel}
\usepackage{caption}
\usepackage[toc,page,title,titletoc]{appendix}

\usepackage[ruled]{algorithm2e}
\usepackage{bbm}
\usepackage{lmodern}
\usepackage{multirow} 
\usepackage{array}    

\usepackage{booktabs} 
\usepackage{array}    
\usepackage{xcolor}   
\usepackage{diagbox}
\usepackage{xfrac} 
\usepackage{caption} 
\usepackage[english]{babel}
\usepackage[autostyle, english = american]{csquotes}
\MakeOuterQuote{"}
\usepackage{tocloft}
\usepackage{titlesec}
\usepackage{enumitem}
\usepackage{accents}
\usepackage{endnotes}

\titleformat{\subsubsection}[block]{\normalfont\bfseries}{\thesubsubsection}{1em}{}
\titlespacing*{\subsubsection}{0pt}{1em}{1em}  

\def \cA {\mathcal{A}}
\def \cB {\mathcal{B}}
\def \cF {\mathcal{F}}
\def \cFl {\mathcal{F}_{\text{Lip}}}

\def \cH {\mathcal{H}}
\def \sumtj {\sum_{t\in\cT_j}}

\def \sumt {\sum_{t=1}^T}
\def \cT {\mathcal{T}}
\def \cS {\mathcal{S}}

\def \lt {\left}
\def \rt {\right}
\def \la {\left\langle}
\def \ra {\right\rangle}

\DeclarePairedDelimiter\floor{\lfloor}{\rfloor}
\def \ind {\mathbbm{1}}

\def \lt {\left}
\def \rt {\right}
\def \eps {\epsilon}
\def \tO {\tilde{O}}

\def \E {\mathbb{E}}

\def \mo {\mathbf{1}}

\def \tF {\tilde{\cF}}

\def \cT {\mathcal{T}}

\def \cV {\mathcal{V}}
\def \cL {\mathcal{L}}
\def \cB {\mathcal{B}}
\def \cP {\mathcal{P}}

\def \cX {\mathcal{X}}
\def \rone {\text{\uppercase\expandafter{\romannumeral1}}\xspace}
\def \rtwo {\text{\uppercase\expandafter{\romannumeral2}}\xspace}
\def \rthe {\text{\uppercase\expandafter{\romannumeral3}}\xspace}

\begin{document}

\TITLE{Learning to Bid in Non-Stationary Repeated First-Price Auctions}

\ARTICLEAUTHORS{%
\AUTHOR{Zihao Hu$^{1,3}$, Xiaoyu Fan$^2$, Yuan Yao$^1$, Jiheng Zhang$^{1,3}$, and Zhengyuan Zhou$^2$}
 \AFF{Department of Mathematics, The Hong Kong University of Science and Technology$^1$ \\
 Stern School of Business, New York University$^2$\\
 Department of IEDA, The Hong Kong University of Science and Technology$^3$\\ 
 \EMAIL{\{zihaohu, yuany, jiheng\}@ust.hk, \{fx2087,zz26\}@stern.nyu.edu}}
}

\ABSTRACT{

    First-price auctions have recently gained significant traction in digital advertising markets, exemplified by Google's transition from second-price to first-price auctions. Unlike in second-price auctions, where bidding one's private valuation is a dominant strategy, determining an optimal bidding strategy in first-price auctions is more complex. From a learning perspective, the learner (a specific bidder) can interact with the environment (other bidders, i.e., opponents) sequentially to infer their behaviors. Existing research often assumes specific environmental conditions and benchmarks performance against the best fixed policy (static benchmark). While this approach ensures strong learning guarantees, the static benchmark can deviate significantly from the optimal strategy in environments with even mild non-stationarity. To address such scenarios, a dynamic benchmark—representing the sum of the highest achievable rewards at each time step—offers a more suitable objective. However, achieving no-regret learning with respect to the dynamic benchmark requires additional constraints. By inspecting reward functions in online first-price auctions, we introduce two metrics to quantify the \emph{regularity} of the sequence of opponents' highest bids, which serve as measures of non-stationarity. We provide a minimax-optimal characterization of the dynamic regret for the class of sequences of opponents' highest bids that satisfy either of these regularity constraints. Our main technical tool is the Optimistic Mirror Descent (OMD) framework with a novel optimism configuration, which is well-suited for achieving minimax-optimal dynamic regret rates in this context. We then use synthetic datasets to validate our theoretical guarantees and demonstrate that our methods outperform existing ones.


}


\maketitle

\section{Introduction}


By 2029, global digital advertising spending is projected to reach $1126$ billion \citep{statistaad}. As online ad display grows in importance, it has become a central focus in operations research, information systems, and machine learning (see e.g., \citealp{wang2017display,choi2020online}). In online ad markets (also known as ad exchanges), advertisers bid for ad impressions offered by publishers on ad exchanges through auctions to maximize their rewards, while publishers manage inventory to optimize customer impressions. Specifically, in each auction round, the publisher displays an ad impression to potential advertisers (buyers), who assess its value and submit bids. Allocation and pricing of impressions are then determined by an online auction protocol.

Second-price auctions, championed by Nobel-prize-winning work of \citet{vickrey1961counterspeculation}, have been widely used in online ad markets \citep{edelman2007internet,despotakis2021first} for their incentive compatibility, which encourages truthful bidding. In this format, the highest bidder wins the ad impression but pays the second-highest bid. Despite its theoretical elegance, second-price auctions face practical criticisms, particularly the potential for auctioneers to manipulate the second-highest bid to inflate payments undetectably \citep{rothkopf1990vickrey,lucking2000vickrey,akbarpour2020credible}. In online ad auctions, such manipulation allows ad exchanges to substantially increase revenue. Concerns over trust and the rise of publisher-initiated header bidding \citep{despotakis2021first} have led major ad exchanges—including Google AdSense \citep{wong2021moving}, Google Ad Manager \citep{bigler2019rolling}, Yahoo Advertising \citep{alcobendas2021adjustment}, and Xandr \citep{xandr2024auction}—to shift to first-price auctions. In these, the highest bidder wins and pays their bid, thereby addressing trust issues. However, first-price auctions lack incentive compatibility, as revealing a bidder's true valuation is no longer optimal. This raises a critical question: what bidding strategies should bidders adopt in online first-price auctions to maximize cumulative rewards?

Two primary perspectives address this problem: the game-theoretic and online learning approaches. From the game-theoretic perspective, the problem originates from foundational work by \citet{vickrey1961counterspeculation,myerson1981optimal}, where bidders are modeled as rational agents with partial or complete information about competitors' valuation distributions. This framework allows derivation of optimal strategies and Bayesian Nash equilibria. While significant progress has been made in computing Bayesian Nash equilibria for online first-price auctions \citep{wang2020bayesian,filos2021complexity,bichler2023computing,chen2023complexity,filos2024computation}, these methods often assume bidders have some precise knowledge of valuation distributions. While such assumptions may hold in physical auctions—where industry peers have insights into each other's valuations—they are less realistic in online auctions, where bidders typically lack information about competitors' identities, making valuation estimation far more challenging.

An alternative is the online learning perspective, where a specific bidder is treated as the "learner" and the remaining bidders are modeled as the "environment" (potentially with some assumptions on the environment to ensure learnability). In this view, the problem of finding the optimal bidding strategy can be cast as a sequential two-player game. At the start of the game, the learner is assumed to have no knowledge of the environment. However, based on past decisions and the feedback received, the learner can iteratively update their bidding strategy. A common performance metric in this perspective is called (static) regret, which measures the difference between the cumulative reward achieved by the best fixed policy and the cumulative reward of the learning algorithm. The goal of online learning is to achieve sublinear regret, which ensures that the time-averaged performance of the learning algorithm asymptotically converges to that of the best fixed policy.
This perspective has inspired a body of seminal work \citep{han2020learning,zhang2022leveraging,badanidiyuru2023learning,balseiro2023contextual,han2025optimal} focused on achieving sublinear regret in both stochastic and adversarial settings. In these contexts, the private value of the bidder and/or the opponents' highest bid at each time step are either independently and identically distributed (i.i.d.) or adversarially generated.


While these approaches provide strong theoretical guarantees, real-world scenarios are usually much more involved. The value of a fixed ad impression can change over time, possibly exhibiting seasonal periodic trends or even sudden shifts due to unforeseen events. The values of different ad impressions can be correlated in complex ways, since advertisers may have complementary marketing campaigns, competing objectives, or overlapping target audiences.
 For a certain learner, their opponents' bidding strategies may also evolve over time, as they adapt to the learner's bidding behavior. These complexities often fall outside the assumptions of purely stochastic or adversarial environments, and it is more natural to term these situations as \emph{non-stationary environments}. In such environments, considering competing with the best fixed policy (static benchmark) might oversimplify the situation that the learner faces.
In contrast, a dynamic benchmark—representing the maximum achievable cumulative reward—is always optimal, even in non-stationary settings.
Learning in non-stationary environments poses a fundamental challenge in operations research \citep{besbes2015non,besbes2019optimal,cheung2022hedging,cheung2023nonstationary} and machine learning \citep{yang2016tracking,zhang2018adaptive,wei2021non}, yet it remains underexplored in the context of online first-price auctions. This gap motivates the core focus of this work, which investigates the following key questions:

\emph{For online first-price auctions in non-stationary ad markets, can we effectively compete with a dynamic benchmark? What mathematical tools can help establish minimax-optimal dynamic regret rates in such settings?}

\subsection{Our Contributions}


Consider a learner participating in a $T$-round online first-price auction. In each round $t$, the learner observes an ad impression, receives a private valuation $v_t$, and then determines a bid $b_t$. After submitting $b_t$, the learner observes $m_t$, the highest bid from other participants, and receives a reward of $(v_t-b_t)\cdot\ind(b_t\ge m_t)$. This full-information feedback setting is widely used in practice, including Google Ad Exchange \citep{google_openrtb_adx_proto}. Other feasible feedback types include binary feedback \citep{balseiro2023contextual} and winning-bid feedback \citep{han2025optimal}. Our main contributions are summarized as follows:
\begin{itemize}

    \item Inspired by previous work on non-stationary online learning \citep{besbes2015non,jadbabaie2015online,yang2016tracking}, we propose two regularity conditions on the sequence of opponents' highest bids to characterize the extent of non-stationarity: the temporal variation $V_T\coloneqq\sum_{t=2}^T|m_t-m_{t-1}|$ and the number of abrupt switches $L_T\coloneqq\sum_{t=2}^T\ind(m_t\neq m_{t-1})$. If either $V_T=\Omega(T)$ or $L_T=\Omega(T)$, then any non-anticipatory policy suffers $\Omega(T)$ dynamic regret. Thus, a reasonable goal is to achieve sublinear dynamic regret rates when $V_T=o(T)$ or $L_T=o(T)$. Notably, these regularity conditions do not depend on the learner's private valuation sequence, allowing the learner's private valuation sequence to be adversarially generated.

    \item We propose policies that are efficiently implementable and achieve the minimax-optimal dynamic regret guarantees of $\tO\left(\sqrt{TV_T}\right)$ and $\tO\left(L_T\right)$, where $\tO(\cdot)$ hides poly-logarithmic factors. The one-sided Lipschitzness of the reward function poses significant challenges in predicting the optimal bid, as discussed further in Section \ref{sec:full_why}. To address this challenge, we employ the Optimistic Mirror Descent (OMD) framework \citep{chiang2012online,rakhlin2013optimization}, a powerful tool with a customizable optimism vector that achieves improved static regret rates in slowly evolving environments. Interestingly, we design the optimism not simply to minimize the static regret, but rather to achieve a favorable balance between the static regret and the transition cost from static regret to dynamic regret, thus achieving the optimal dynamic regret rates.

    \item We establish $\Omega(\sqrt{TV_T})$ and $\Omega(\sqrt{L_T})$ minimax lower bounds for online first-price auction instances regularized by either $V_T$ or $L_T$, respectively. For sequential learning of convex and Lipschitz functions with exact feedback, the dynamic regret lower bound in terms of $V_T$ is $\Omega(V_T)$ \citep{jadbabaie2015online,yang2016tracking}. Our results, therefore, highlight a sharp separation between learning one-sided Lipschitz functions and convex, Lipschitz functions. To prove the lower bounds, we construct batches with small temporal variations. Within each batch, the optimal dynamic regret of any non-anticipatory policy can be computed via dynamic programming. By suitably concatenating these batches, we derive the desired lower bounds.     
    
    \item Since both $V_T$ and $L_T$ capture different types of regularity in the opponents' highest bid sequence, it is desirable to achieve a best-of-both-worlds guarantee of $\tilde{O}(\min\{\sqrt{TV_T}, L_T\})$, which automatically adapts to the better of the two bounds.
 We achieve this by combining our algorithms using the meta algorithm of \citet{sani2014exploiting}. The theoretical guarantees are summarized in Table~\ref{tab:res_full}. Notably, the lower bounds hold even if $V_T$ or $L_T$ is known in advance, while the upper bounds do not require such prior knowledge.

 \begin{table}[h!]
        \captionsetup{justification=raggedright, singlelinecheck=false}
        \caption{Dynamic regret rates lower bounds and upper bounds when the either $V_T$ or $L_T$ is constrained. $V_T\coloneqq\sum_{t=2}^T|m_t-m_{t-1}|$ and $L_T\coloneqq\sum_{t=2}^T\ind(m_t\neq m_{t-1})$ are two metrics to measure the regularity of the opponents' highest bid sequence. Here we use $\tO(\cdot)$ to omit polylogarithmic factors.}
        \centering
        \renewcommand{\arraystretch}{1.5} 
        \setlength{\tabcolsep}{10pt}      
        \begin{tabular}{|c|c| c |}
            \hline
            Regularity & Upper Bound &  Lower Bound    \\ 
            \hline\hline 
           {$V_T$}  & $\tO(\sqrt{TV_T})$, Theorem \ref{thm:full_vt_ub}  &$\Omega(\sqrt{TV_T})$, Theorem \ref{thm:full_lb}  \\
            \hline
            $L_T$  & $\tO(L_T)$, Theorem \ref{thm:full_lt1} & $\Omega(L_T)$, Theorem \ref{thm:full_lt2}  \\
            \hline
            Best-of-both-worlds & $\tO(\min\{\sqrt{TV_T},L_T\})$, Theorem \ref{thm:bobw} & -\\
            \hline
        \end{tabular}
        \label{tab:res_full}
    \end{table}

    \item We evaluate our theoretical findings using synthetic datasets. We first confirm that our algorithms achieve the theoretical dynamic regret rates. We then consider a multi-agent bidding environments where the opponents of the learner run the budget-pacing policy by \citet{gaitonde2022budget}, and demonstrate that our algorithms outperform two important baselines: the Hedge algorithm and the SEW policy \citep{han2020learning}, especially in the regime where opponents have limited budgets.

\end{itemize}
\subsection{Key Challenges}

The primary challenge in proving the upper bounds stems from the one-sided Lipschitz property of the reward function in online first-price auctions. In simple terms, bidding slightly higher than necessary results in only a minor revenue loss, while bidding slightly lower can cause a much larger loss. To address non-stationarity in the data, we adopt the restart scheme introduced by \citet{besbes2015non}—dividing the time horizon $T$ into batches and restarting a dedicated algorithm at the beginning of each batch. Our analysis decomposes the dynamic regret into two components: the static regret and the transition cost that bridges static and dynamic regret. Unlike previous work, we address these two terms using novel analytical tools. 
For instance, to bound the transition cost as in \citet{besbes2015non}, one needs to bound the temporal variation of the reward sequence by that of the opponents' highest bid sequence, which roughly amounts to bounding the variation of rewards by the variation of the maximizers\footnote{The correspondence is not exact: when $v_t \geq m_t$, the opponents' highest bid $m_t$ is the reward maximizer, but when $v_t < m_t$, the reward is maximized for any bid smaller than $m_t$.}. While previous literature \citep{jadbabaie2015online,yang2016tracking} assumes full Lipschitzness, we relax this assumption by relying solely on the one-sided Lipschitz property.



To control the static regret, we recall that the SEW policy \citep{han2020learning} achieves the minimax-optimal $\tilde{O}(\sqrt{T})$ static regret bound \citep{han2025optimal} in online first-price auctions with adversarial inputs. Moreover, in settings with convex Lipschitz loss functions and noisy feedback, \citet{besbes2015non} demonstrated that restarting the online gradient descent (OGD) algorithm, which achieves minimax-optimal static regret, ensures minimax-optimal dynamic regret. Surprisingly, directly plugging the SEW policy into the restart scheme does not produce the same optimal dynamic regret rates—
likely because the SEW policy, designed for adversarial environments, lacks adaptivity in slowly varying settings.

To improve the sublinear static regret guarantee under slow variation, we employ Optimistic Mirror Descent (OMD) \citep{chiang2012online,rakhlin2013optimization,wei2018more}—a variant of mirror descent that incorporates an “optimistic” guess of the gradient of the expected reward for the current round. When this guess is taken as the gradient from the previous round, OMD can yield much lower static regret than $O(\sqrt{T})$ in slowly varying environments. 
Yet, no previous work has used OMD to reach minimax-optimal dynamic regret rates. By carefully configuring the optimism vector, we show that OMD’s static regret can be bounded by the transition cost plus a small additive term, effectively balancing the trade-off between static regret and transition cost and leading to minimax-optimal dynamic regret rates.

The difficulty in proving the lower bounds arises because the learner directly observes the highest bids of the others rather than receiving noisy feedback. Noisy feedback simplifies many information-theoretic arguments—such as those based on Le Cam’s method—which have been crucial for deriving lower bounds in non-stationary online learning \citep{besbes2015non,besbes2019optimal,cheung2022hedging} and in online first-price auctions \citep{han2020learning,zhang2022leveraging,cesa2024role}. In our setting, alternative approaches are required. Our optimal lower bounds are inspired by the minimax lower bounds for learning with a few experts \citep{cover1966behavior,gravin2016towards,harvey2023optimal}, which establish lower bounds by constructing suitable problem instances whose minimax value can be evaluated. Specifically, leveraging the one-sided Lipschitz property, we construct batches of opponents' highest bids with small temporal variation, but any non-anticipatory policy suffers a large amount of dynamic regret in these batches. We then carefully stitch these batches together to obtain the desired lower bound results.





\subsection{Paper Organization}
The paper is organized as follows. In Section \ref{sec:related_work}, we review prior work, positioning our contributions within the existing literature. Section \ref{sec:problem_form} formally defines the problem setting and introduces our methodology. In Section \ref{sec:full_ub}, we present our upper bound results by deriving dynamic regret rates within the Optimistic Mirror Descent (OMD) framework. Section \ref{sec:full_lb} then provides lower bounds via a minimax analysis, thereby establishing the optimality of our upper bounds. Section \ref{sec:experiment} offers numerical simulations that validate the dynamic regret rates of our proposed algorithms and compare them with baseline approaches in a multi-agent bidding environment. Finally, Section \ref{sec:conclusion} summarizes our findings and discusses potential directions for future research.

\subsection{Notations}

Let $v_t$ and $m_t$ denote the learner's private valuation and the highest bid from other bidders at round $t$, respectively. We denote the learner's bid at round $t$ by $b_t$. Following previous work \citep{han2020learning,balseiro2023contextual,han2025optimal}, we assume $v_t, m_t, b_t \in [0,1]$. Since the other bidders' highest bid $m_t$ is observed by the learner, we discretize the decision space $[0,1]$ into $N$ discrete bidding prices and model the problem as \emph{learning with expert advice} \citep{cesa2006prediction} with $N$ experts. Let $r_{t,i}$ denote the reward of the $i$-th expert at round $t$.


Additionally, $\ind(\cdot)$ denotes the indicator function of an event.
$\mathbb{E}[\cdot]$ represents the expectation operator.
$[s] \coloneqq \{1, \dots, s\}$ denotes the set of integers from 1 to $s$. For a convex and differentiable function $\psi$ defined on a convex region $\mathcal{P}$, $D_\psi(p,q)\coloneqq \psi(p)-\psi(q)-\langle p-q,\nabla\psi(q)\rangle$ is the Bregman divergence. We use $\mathbf{1}$ to denote an all-ones vector. We use standard asymptotic notation $O(\cdot)$, $\Omega(\cdot)$, $\Theta(\cdot)$ and $\tilde{O}(\cdot)$ to simplify the analysis: We use $x_n = O(y_n)$ to denote that there exist constants $n_0 \in \mathbb{N}^+$ and $M \in \mathbb{R}^+$ such that for all $n \geq n_0$, $x_n \leq M \cdot y_n$.
Similarly, $x_n = \Omega(y_n)$ is equivalent to $y_n = O(x_n)$, $x_n=\Theta(y_n)$ means $x_n=O(y_n)$ and $x_n=\Omega(y_n)$, and $\tilde{O}(\cdot)$ is similar to $O(\cdot)$ but hides polylogarithmic factors.


\section{Related Work}
\label{sec:related_work}
In this section, we briefly review relevant work on first-price auctions and online learning in non-stationary environments.

\paragraph{First-price Auctions.} Although Vickrey is more commonly associated with the second-price auction, \citet{vickrey1961counterspeculation} formalize and compare several auction formats, including the first-price auction. In recent years, as certain online ad exchanges switch from second-price to first-price auctions, first-price auctions gain increasing attention from researchers in economics, operations research, and machine learning. From a game-theoretic perspective, researchers study aspects such as the Bayesian Nash equilibrium, pacing equilibrium, and algorithmic collusion behaviors in first-price auctions \citep{wang2020bayesian,filos2021complexity,conitzer2022pacing,banchio2022artificial,banchio2023adaptive,chen2023complexity,bichler2023computing,jin2023first,balseiro2023contextual}.

This work focuses on a learning perspective, where a learner sequentially interacts with the environment to learn an optimal bidding strategy. Inspired by patterns in real-world auction data, \citet{zhang2021meow} introduce a non-parametric approach for bid updates, demonstrating its superiority over traditional parametric methods. \citet{balseiro2023contextual} employ cross-learning to improve regret rates for online first-price auctions with binary feedback. When $v_t$ is i.i.d. from a known distribution and $m_t$ is chosen adversarially, they achieve a regret rate of $\tO(T^{\frac{2}{3}})$, improving upon the $\tO(T^{\frac{3}{4}})$ rate achieved by standard contextual bandit techniques. Later, \citet{schneider2024optimal} extend these results to the setting where the distribution of $v_t$ is unknown, achieving the same $\tO(T^{\frac{2}{3}})$ regret through novel techniques. \citet{han2020learning} study online first-price auctions with full-information feedback when both $v_t$ and $m_t$ are chosen adversarially. Using the tree-chaining technique \citep{cesa2017algorithmic}, they achieve a regret rate of $\tO(\sqrt{T})$ against the set of 1-Lipschitz policies. When $m_t$'s are i.i.d. generated, \citet{han2025optimal} improve the analysis of \citet{balseiro2023contextual} to the winning-bid feedback setting throught some novel observations, demonstrating $\tO(\sqrt{T})$ regret. Additionally, \citet{zhang2022leveraging} explore improved regret guarantees by incorporating hints about bidding profiles. \citet{badanidiyuru2023learning} consider online first-price auctions where $m_t$ is generated by a context vector with log-concave noise and establish $\tO(\sqrt{T})$ regret guarantees under full-information feedback. \citet{wang2023learning} investigate first-price auctions with budget constraints, achieving sublinear regret rates when both $v_t$ and $m_t$ are i.i.d. \citet{kumar2024strategically} study settings where $v_t$ is i.i.d. and $m_t$ is adversarially chosen, achieving $\tO(\sqrt{T})$ regret that is both rate-optimal and strategically robust. \citet{cesa2024role} characterize minimax-optimal static regret rates for various feedback settings, highlighting the role of auction format transparency. All the aforementioned works focus on competing against the best fixed policy within a pre-determined policy set, whereas our work aims to compete with the policy that achieves the maximum possible revenue.

Recently, there is a growing body of work on online first-price auctions, where the learner faces additional constraints such as budget or ROI constraints. 
\citet{balseiro2019learning} propose the budget-pacing dynamics in online second-price auctions, which use a sequence of Lagrangian multipliers to shade the learner's bid. 
\citet{gaitonde2022budget} generalize this idea to online first-price auctions with budgets, while 
\citet{lucier2024autobidders} further allow the existence of ROI constraints. Other related work includes \citet{ai2022no,castiglioni2022online,wang2023learning,fikioris2023liquid,aggarwal2025no}. Both \citet{gaitonde2022budget} and \citet{lucier2024autobidders} consider to compete with a dynamic benchmark as well, but there are some key differences between their work and ours. First, \citet{gaitonde2022budget,lucier2024autobidders} consider value maximizing bidders while our work considers revenue maximizing bidders; the value maximizing bidders make sense in the constrained setting but not for the unconstrained setting, since the bidder has the incentive to win every ad impression. Second, \citet{gaitonde2022budget,lucier2024autobidders} consider competing with a sequence of Lagrangian multipliers, where each Lagrangian multiplier makes the expected expenditure at that round equal to the ratio of the initial budget and the time horizon. This sequence is not guaranteed to achieve the highest possible cumulative value, so the dynamic benchmark considered therein is weaker than ours. Third, it is unknown whether \citet{gaitonde2022budget,lucier2024autobidders} achieve minimax-optimal regret rates even with respect to this relaxed dynamic benchmark notion.

\paragraph{Learning in Non-stationary Environments.} \citet{besbes2015non} study stochastic optimization in non-stationary environments, where the loss at each round may vary, and show that sublinear dynamic regret is achievable when the temporal variation—a measure of the total change in the loss function over time—is sublinear in the time horizon. \citet{besbes2015non} provide minimax-optimal characterizations of dynamic regret for online convex optimization and bandit convex optimization. Our problem formulation is inspired by \citet{besbes2015non}, but we make necessary adjustments to better accommodate the one-sided Lipschitzness of the reward function. Please see Remark \ref{rmk:defn} in Section \ref{sec:problem_formulation} for a comprehensive comparison.
\citet{besbes2015non} assumes that the temporal variation of the loss sequence is known in advance. \citet{jadbabaie2015online} demonstrate how to remove this assumption in the online convex optimization setting. Additionally, \citet{jadbabaie2015online,yang2016tracking,zhang2018adaptive,baby2021optimal,baby2022optimal} explore alternative definitions of dynamic regret, such as the path-length of the minimizers of the loss functions, and establish corresponding dynamic regret guarantees. These formulations are also related to ours, but these approaches rely on strong convexity, exponential concavity, convexity, Lipschitzness or smoothness of the loss functions, while the reward function we consider is merely one-sided Lipschitz. \citet{besbes2019optimal} investigate multi-armed bandit problems under non-stationary reward distributions, demonstrating that sublinear regret can be achieved if the total variation of these distributions is known and sublinear in the time horizon. To remove the need for prior knowledge of the variation budget, \citet{cheung2022hedging} propose the bandit-over-bandit technique, which applies to various non-stationary stochastic bandit problems. Building on this, \citet{zhao2021simple} simplify the analysis in \citet{cheung2022hedging} and derive sublinear regret bounds for linear bandits with variable decision sets. In the context of reinforcement learning (RL), \citet{cheung2023nonstationary} employ a similar bandit-over-RL approach to tackle non-stationary settings, achieving nearly optimal regret bounds. \citet{wei2021non} provide a general framework for non-stationary online learning, covering both linear bandits and RL, and achieve optimal dynamic regret rates. \citet{simchi2023non} study experimental design under non-stationary linear trends, while \citet{chen2025learning} focus on non-stationary multi-armed bandits with periodic mean rewards. \citet{huang2025stability} consider non-stationary online learning with noisy realization
of the losses, and achieve minimax-optimal regret guarantees when losses are strongly convex or merely Lipschitz. Though \citet{haoyu2020online} study online second-price auctions in non-stationary settings, their objective and methods differ significantly from ours.


\section{Problem Formulation and Main Results}
\label{sec:problem_form}
In this section, we introduce the problem formulation for online first-price auctions in non-stationary environments, outline the main algorithmic framework we will use, present the informal main results, and define the notations that will be used throughout the paper.


\subsection{Problem Formulation}
\label{sec:problem_formulation}
In non-stationary environments, an advertiser's valuation for an ad impression can vary over time, requiring advertisers to account for this variability when participating in online first-price auctions. We begin with a general description of the online first-price auction \citep{han2020learning,han2025optimal,cesa2024role}, followed by a formal definition of the dynamic benchmark and possible ways to quantify the degree of non-stationarity.

In this auction format, a set of bidders (advertisers) competes to purchase ad impressions from a publisher. Each round, the publisher displays an ad impression along with relevant details, such as user demographics, keywords, and the ad's size and location. Each bidder estimates the value of the ad impression and submits a bid. Under the first-price auction protocol, the bidder who offers the highest bid wins the ad impression and pays the bid amount. Formally, the online first-price auction is a game spanning $T$ rounds. In each round $t=1,\dots,T$, the bidder observes an ad impression, generates a private value $v_t\in[0,1]$, and submits a bid $b_t\in[0,1]$. Let $m_t\in[0,1]$
represent the highest bid among other bidders. The bidder's payoff is then given by
\[
r(b_t; v_t, m_t) \coloneqq (v_t - b_t) \cdot \ind(b_t \geq m_t).
\]
Here, $\ind(b_t\ge m_t)$ is the indicator function that equals $1$ if the bidder wins the auction (i.e., $b_t\ge m_t$), and $0$ otherwise. The one-sided Lipschitz property means: when moving from a higher bid $b'\le v_t$ to a slightly lower bid $b$, the reward increase is bounded by the bid difference, but the reward decrease can be significant (particularly when crossing the discontinuity at $m_t$).

For simplicity, we assume the time horizon $T$ is known to the learner. If 
$T$ is unknown, the doubling trick \citep{auer2002adaptive,cesa2006prediction} can be used to eliminate this requirement. Since this is a sequential decision-making problem, it is essential to formally define the information received by the learner before submitting $b_t$. We mainly consider the case where the learner observes $m_t$, the highest bid from other bidders, so the information up to time $t-1$ can be described by the following filtration:
\[
    \cH_t\coloneqq\sigma((v_s,m_s)_{s=1}^{t-1},v_t),
\]
where $\sigma(\cdot)$ is the $\sigma$-algebra generated by the observations. Conventionally, the filtration up to \(t-1\) should not include \(v_t\), as this represents information from the current round. However, prior work \citep{han2020learning,balseiro2023contextual,han2025optimal} assumes that the bidder knows \(v_t\) before determining their bid \(b_t\). This assumption is reasonable because ad exchanges typically display the ad impression and related contextual or demographic information to bidders, enabling them to estimate the value of the impression. Therefore, we include \(v_t\) in the filtration.

Let $(\Omega, \mathcal{F}, \mathbb{P})$ be a probability space with sample space $\Omega = [0,1]$ and $\sigma$-algebra $\mathcal{F}$. Let $U$ be a random variable defined on this probability space, i.e., $U: \Omega \to \mathbb{R}$ is an $\mathcal{F}$-measurable function. We define the set of admissible policies $\Pi$ as follows:

For each $t \in [T] = {1, 2, \dots, T}$, let $\pi_t: \mathbb{R}^{2t-1} \times \Omega \to \mathbb{R}$ be a measurable function. A policy $\pi \in \Pi$ is then a sequence of such measurable functions: $\pi = (\pi_1, \pi_2, \dots, \pi_T)$. Given a policy $\pi \in \Pi$, the bid $b_t$ at time $t$ is determined by:
\[
b_t=\pi_t((v_s,m_s)_{s=1}^{t-1},v_t,U).
\]
Thus, the set of admissible policies $\Pi$ is characterized by the collection of these measurable functions $\{\pi_t\}_{t=1}^T$. Note that the probability measure $\mathbb{P}$ plays a crucial role in determining the distribution of the random variable $U$ and consequently influences the stochasticity of the bidding process.


Previous work on online first-price auctions typically aims to achieve sublinear regret over \(T\) rounds against the best fixed policy in hindsight. Formally, this involves designing a policy to minimize the regret:  
\[
\E[\text{R}_T(\pi)] \coloneqq \sup_{f \in \tilde{\cF}} \sum_{t=1}^T r(f(v_t); v_t, m_t) - \sumt\E[r(b_t; v_t, m_t)],
\]
where \(\tilde{\cF}\) is a class of policies. Common choices for \(\tilde{\cF}\) include the set of \(1\)-Lipschitz policies \citep{han2020learning}, the set of monotone policies \citep{han2020learning} or the set of policies that map \(C\) possible valuations to \(K\) discrete bids \citep{balseiro2023contextual,schneider2024optimal}.  

Here, we refer to $\sup_{f \in \tilde{\cF}} \sum_{t=1}^T r(f(v_t); v_t, m_t)$ as the \emph{static benchmark}. In contrast, we define the \emph{dynamic benchmark} as:  
\begin{equation}\label{eq:dynamic}
    \sum_{t=1}^T r(b_t^*; v_t, m_t)\coloneqq\sum_{t=1}^T\max_{b\in[0,1]}r(b;v_t,m_t) = \sum_{t=1}^T \max\{v_t - m_t, 0\},
\end{equation}
where $b_t^*\in\argmax_{b\in[0,1]}r(b;v_t,m_t)$ is the optimal bid at round $t$. We note that $b_t^*$, the optimal bid at round $t$, should be $m_t$ whenever $v_t \geq m_t$, and can be any value no greater than $m_t$ when $v_t < m_t$. Without loss of generality, in this work, we set
\begin{equation}\label{eq:bt_star}
    b_t^*=\begin{cases}
        m_t,&v_t\ge m_t\\
        v_t,&v_t<m_t.
    \end{cases}
\end{equation}
This sequence achieves optimal revenue while eliminating the ambiguity of the optimal bid.

It is immediate to see that the dynamic benchmark represents the maximum possible revenue that the learner can achieve. Moreover, the dynamic benchmark can outperform the static benchmark by \(\Omega(T)\), even in instances of online first-price auctions with mild regularity in the opponents' highest bid sequence. Example \ref{ex:sta2dyn} illustrates the reason for this discrepancy between static and dynamic benchmarks.
\begin{example}\label{ex:sta2dyn}
    Assume $v_t\equiv 1$ for $t\in[T]$ and 
    \[
    m_t=\begin{cases}
        0, &1\le t\le\frac{T}{2},\\
        \frac{1}{2}, &\frac{T}{2}+1\le t\le T.
    \end{cases}
    \]
    Then
    \begin{equation*}
    \begin{split}
        &\sumt r(b_t^*;v_t,m_t)-\sup_{f\in\tilde{\cF}}\sumt r(f(v_t);v_t,m_t)\\
        =&\sumt \max\{v_t-m_t,0\}-\sup_{f\in\tilde{\cF}}\sumt r(f(v_t);v_t,m_t)=\frac{3T}{4}-\frac{T}{2}=\frac{T}{4}.
    \end{split}
    \end{equation*}
    The main fact we rely on is that \(f(v_t) \equiv f(1)\) can only take a single real value and, as such, cannot be optimal on both segments.
\end{example}
Consequently, a no-regret online learning policy, while converging to the best fixed policy in the long run, does not converge to the policy with the highest possible revenue. In this work, we consider minimizing the following dynamic regret in online first-price auctions:
\begin{equation}\label{eq:dynamic_reg}
\mathbb{E}[\text{DR}_T(\pi)] \coloneqq \sum_{t=1}^T r(b_t^*; v_t, m_t) - \sum_{t=1}^T \mathbb{E}[r(b_t; v_t, m_t)].
\end{equation}

It is well-established (e.g., \cite{besbes2015non,yang2016tracking,zhang2018adaptive,besbes2019optimal}) that achieving sublinear dynamic regret uniformly is impossible without imposing further constraints on the problem instances. To ensure no-regret online learning, we investigate policies with sublinear dynamic regret guarantees under the assumption that the regularity of the opponents' highest bid sequence is sublinear in the time horizon $T$. We consider two specific metrics to quantify this regularity:
\begin{align}
    &V_T\coloneqq\sum_{t=2}^T|m_t-m_{t-1}|\label{eq:def_vt}\\
    &L_T\coloneqq\sum_{t=2}^T\ind(m_t\neq m_{t-1})\label{eq:def_lt},
\end{align}
where $V_T$ measures the temporal variation of the opponents' highest bid sequence, while $L_T$ measures the number of abrupt switches in the opponents' highest bid sequence.
\begin{remark}
    \label{rmk:defn}
    
    The regularity conditions on the opponents' highest bid sequence (Equations \eqref{eq:def_vt} and \eqref{eq:def_lt}) are inspired by \citet{besbes2015non}, where the authors use the temporal variation of reward/loss functions as a regularity measure. In our setting, their measure translates to $\sum_{t=2}^T \sup_{b\in[0,1]} |r(b; v_t, m_t) - r(b; v_{t-1}, m_{t-1})|$. 
    However, we argue that $\sum_{t=2}^T|m_t-m_{t-1}|$ is a more compact and reasonable metric for measuring non-stationarity. By Proposition \ref{prop:compare}, $\sum_{t=2}^T|m_t-m_{t-1}|$ is at most twice $\sum_{t=2}^T \sup_{b\in[0,1]} |r(b; v_t, m_t) - r(b; v_{t-1}, m_{t-1})|$. In general, however, $\sum_{t=2}^T \sup_{b\in[0,1]} |r(b; v_t, m_t) - r(b; v_{t-1}, m_{t-1})|$ can be much larger than $\sum_{t=2}^T|m_t-m_{t-1}|$, as demonstrated in Examples \ref{ex:example1} and \ref{ex:example2}. The disadvantages of $\sum_{t=2}^T \sup_{b\in[0,1]} |r(b; v_t, m_t) - r(b; v_{t-1}, m_{t-1})|$ stem from: (i) this metric neglects the one-sided Lipschitzness of the reward function; (ii) this metric depends on the sequence $(v_t)_{t=1}^T$, which is unnecessary upon careful inspection.


In contrast, our measure defined in Equation \eqref{eq:def_vt} compactly captures the regularity of the opponents' highest bid sequence while avoiding both disadvantages. Additionally, \citet[Figure 1]{besbes2015non} emphasize two types of temporal patterns: continuous change and discrete shocks, which directly correspond to our regularity conditions in Equations \eqref{eq:def_vt} and \eqref{eq:def_lt}, respectively. 

\end{remark}

\begin{proposition}\label{prop:compare}
    For any $v_{t-1},v_t,m_{t-1},m_t\in[0,1]$,
    \[
    \begin{split}        
    |m_t-m_{t-1}|\le& 2\sup_{b\in[0,1]} |r(b; v_t, m_t) - r(b; v_{t-1}, m_{t-1})|.\\
    \end{split}
    \]
\end{proposition}

\begin{example}\label{ex:example1}
    Assume $v_t\equiv 1$ for $t\in[T]$, and 
    \[
    m_t=\begin{cases}
        0,\quad t\text{ is odd}\\
        \epsilon,\quad t\text{ is even}.
    \end{cases}
    \]
    Then $\sum_{t=2}^T|m_t-m_{t-1}|=(T-1)\epsilon$ while $\sum_{t=2}^T \sup_{b\in[0,1]} |r(b; v_t, m_t) - r(b; v_{t-1}, m_{t-1})|=T-1$.
\end{example}
\begin{example}\label{ex:example2}
    Assume $m_t\equiv c$ for $t\in[T]$ where $c\in[0,1]$, and
    \[
    v_t=\begin{cases}
        0,\quad t\text{ is odd}\\
        1,\quad t\text{ is even}.
    \end{cases}
    \]
    Then $\sum_{t=2}^T|m_t-m_{t-1}|=0$ while $\sum_{t=2}^T\sup_{b\in[0,1]} |r(b; v_t, m_t) - r(b; v_{t-1}, m_{t-1})|=T-1$.
\end{example}
We aim to establish bounds on the dynamic regret rates under two different regularity conditions. Formally, we consider the suprema of the expected dynamic regret over two sets of feasible opponents' highest bid sequences $\sup_{(v_t,m_t)_{t=1}^T\in\cV}\E[\text{DR}_T(\pi)]$ and $\sup_{(v_t,m_t)_{t=1}^T\in\cL}\E[\text{DR}_T(\pi)]$
where the sets $\mathcal{V}$ and $\mathcal{L}$ are defined as follows:
\[
    \cV=\lt\{\{(v_t,m_t)_{t=1}^T\}:\sum_{t=2}^T|m_t-m_{t-1}|\le V_T\rt\},\quad\cL=\lt\{\{(v_t,m_t)_{t=1}^T\}:\sum_{t=2}^T\ind(m_t\neq m_{t-1})\le L_T\rt\}.
    \]
    Here, $\mathcal{V}$ represents the set of opponents' highest bid sequences with variation bounded by $V_T$, while $\mathcal{L}$ represents the set of opponents' highest bid sequences with a limited number of changes, bounded by $L_T$. Before establishing dynamic regret rates, we first present a result that highlights the necessity of assuming sublinear regularity in the time horizon. 

    \begin{proposition}\label{prop:lb1}
        Assume $c_1\in\lt[0,\frac{1}{2}\rt]$, then 
        \begin{itemize}
            \item $V_T\ge c_1T$ implies 
            \[
                \inf_{\pi\in\Pi}\sup_{(v_t,m_t)_{t=1}^T\in \cV}\E\lt[\textnormal{DR}_T(\pi)\rt]\ge c_1T
            \]
            holds for any admissible policy.
            \item $L_T\ge c_1T$ implies 
            \[
                \inf_{\pi\in\Pi}\sup_{(v_t,m_t)_{t=1}^T\in \cL}\E\lt[\textnormal{DR}_T(\pi)\rt]\ge c_1^2T
            \]
            holds for any admissible policy.
        \end{itemize}
    \end{proposition}
    
    Based on Proposition \ref{prop:lb1}, a reasonable objective is to achieve sublinear dynamic regret guarantees when either $V_T = o(T)$ or $L_T = o(T)$. We establish the corresponding  upper bounds and lower bounds in Sections  \ref{sec:full_ub} and \ref{sec:full_lb}, respectively.


\subsection{The Optimistic Mirror Descent Framework}
\label{sec:omd}

For online first-price auctions, the learner intends to determine a bid $b_t\in[0,1]$ for each round $t$. For the convenience of algorithmic implementation, we discretize the interval $[0,1]$ into $N$ discrete candidate bidding prices and maintain a probability distribution $p_t$ that governs the probability of selecting the $i$-th discrete bidding price. We can then dynamically adjust the probability of these prices based on their historical performance. Readers familiar with online learning will recognize that we are considering the learning with expert advice framework \citep{cesa2006prediction}, where each expert suggests a potential bidding price.


A key challenge in non-stationary online learning is that the reward sequence may exhibit continuous drifts or abrupt shifts, so the learner might want to adapt more to the local trend. Our algorithms are composed of two ingredients: a restart scheme, pioneered by \citet{besbes2015non}, which decomposes the time horizon into batches satisfying certain criteria; and a static regret minimizer, which is applied to each batch. The static regret minimizer we use can be considered as instantiations of the Optimistic Mirror Descent (OMD) framework developed by \citet{chiang2012online,rakhlin2013optimization,syrgkanis2015fast}.

The OMD framework (as shown in Algorithm \ref{alg:omd}) is a two-stage online mirror descent algorithm. At the beginning of each round $t$, the reward vector $r_t$ is not yet available to the learner, so the learner adopts the first online mirror descent step to incorporate an optimism vector $o_t = o_t(r_1, r_2, \dots, r_{t-1})$ or $\mu_t\cdot\mathbf{1}$, a specific optimism obtained by multiplying a scalar with an all-ones vector, and obtains $p_t$, the probability distribution over $N$ bidding prices. This optimism vector can be considered as the learner's prediction of $r_t$. Of course, the closer $o_t$ is to $r_t$, the smaller static regret rate the learner can achieve. Then the learner chooses $b_t$ by sampling from $p_t$ and receives $r_t$. For the second online mirror descent step, the learner incorporates the actual reward $r_t$ to update the knowledge about the environment, possibly with a second-order correction $a_t$.

For the case where $L_T = o(T)$, we can use Option \rone of Algorithm \ref{alg:omd}, which is reminiscent of the earliest instantiation of OMD \citep{chiang2012online}, where $o_t=r_{t-1}$. Our choice of optimism follows this idea but we incorporate the information of $v_t$ when designing $o_t$ to ensure that the private valuation sequence $(v_t)_{t=1}^T$ does not degrade the regret performance. Also, different from the restart scheme in \citet{besbes2015non} using a fixed batch size, we design an adaptive restart procedure to reduce the transition cost from static regret to dynamic regret, and to achieve the minimax-optimal dynamic regret rate.

\begin{algorithm}[H]
    \SetAlgoLined
    \SetKwInOut{Input}{Input}
    
    \Input{$\cP$ is the convex hull of $\{e_1,\dots,e_N\}$; $\psi(p)$: a convex regularizer defined on the probability simplex. }
    \BlankLine
    $p_1'\leftarrow\argmax_{p\in\cP}-\psi(p)$\;
    \For{$t\leftarrow 1,\dots,T$}{
        Set
        \begin{flushleft}
            \[
        p_t\leftarrow\begin{cases}
            \argmax_{p\in\cP}\lt\{\la p,o_t\ra-D_{\psi}(p,p_t')\rt\}\quad &(\text{Option \rone})\\
            \argmax_{p\in\cP}\lt\{\la p,\mu_t\cdot\mo\ra-D_{\psi}(p,p_t')\rt\}\quad &(\text{Option \rtwo})\\
        \end{cases}
        \]
        \end{flushleft}       
        Choose actions according to $p_t$, receive $r_{t,i}$ for any $i\in[N]$ and set
        \begin{flushleft}
            \[
        a_{t,i}\leftarrow\begin{cases}
            0,\quad&(\text{Option \rone })\\
            4\eta(r_{t,i}-\mu_t)^2,\quad&(\text{Option \rtwo})\\
        \end{cases}
        \]
        \end{flushleft}       
        Update
        \[
            p_{t+1}'\leftarrow\argmax_{p\in\cP}\lt\{\la p,r_t-a_t\ra-D_{\psi}(p,p_t')\rt\}
        \]
    }
    \BlankLine
    \caption{Optimistic Mirror Descent}
    \label{alg:omd}
    \end{algorithm}

The case of $V_T=o(T)$ is more challenging, and we can use Option \rtwo of Algorithm \ref{alg:omd}, which is a variant of OMD by \citet{steinhardt2014adaptivity,wei2018more}. A key feature of Option \rtwo is that the optimism is chosen to be $o_t=\mu_t\cdot\mathbf{1}$, which might appear rigid at first glance. However, due to this choice and since $p$ lies on the probability simplex, we have $\langle p,\mu_t\cdot\mathbf{1}\rangle=\mu_t$, which is constant with respect to $p$. Thus, we obtain $p_t=p_t'$ after the first mirror descent step. Therefore, even though $\mu_t$ depends on $r_t$ $(v_t\text{ and }m_t)$, the variable $p_t$ does not depend on $r_t$, and we indeed comply with the online learning protocol that requires $p_t$ to be chosen before observing $r_t$. Our optimism configuration chooses $\mu_t=\max\{v_t-m_t,0\}$, which is a novel contribution of our work. The $\mu_t$ we choose is not simply targeting the minimization of static regret, but focuses more on achieving a favorable tradeoff between the static regret and the transition cost. This ultimately leads to the minimax-optimal dynamic regret even when using a constant batch size, and we can eliminate the requirement of knowing $V_T$ by employing an adaptive batch size.

    \subsection{Main Results}
    Our main results are summarized as follows:

    \noindent{\scshape{Theorem. }}\textit{(informal) For online first-price auctions,
    \begin{itemize}
        \item consider the set of auction sequences such that $\sum_{t=2}^T|m_t-m_{t-1}|\le V_T$, then one can apply Algorithm \ref{alg:full_vt} to achieve $\tO(\sqrt{TV_T})$ expected dynamic regret (Theorem \ref{thm:full_vt_ub}, Section \ref{sec:full_vt}). Besides, any non-anticipatory policy suffers $\Omega(\sqrt{TV_T})$ expected dynamic regret (Theorem \ref{thm:full_lb}, Section \ref{sec:full_vt_lb}).
        \item consider the set of auction sequences such that $\sum_{t=2}^T\ind(m_t\ne m_{t-1})\le L_T$, then one can apply Algorithm \ref{alg:full_lt} to achieve $\tO(L_T)$ expected dynamic regret (Theorem \ref{thm:full_lt1}, Section \ref{sec:full_lt}). Besides, any non-anticipatory policy suffers $\Omega(L_T)$ expected dynamic regret (Theorem \ref{thm:full_lt2}, Section \ref{sec:full_lt_lb}).
        \item consider an auction instance such that $\sum_{t=2}^T|m_t-m_{t-1}|\le V_T$ and $\sum_{t=2}^T\ind(m_t\ne m_{t-1})\le L_T$, then Algorithm \ref{alg:best} achieves $\tO(\sqrt{TV_T},L_T)$ best-of-both-worlds dynamic regret (Theorem \ref{thm:bobw}, Section \ref{sec:bobw}).
    \end{itemize}
    }

\section{Dynamic Regret Upper Bounds}
\label{sec:full_ub}
In this section, we explore how to achieve minimax-optimal dynamic regret guarantees under the conditions $V_T=o(T)$ or $L_T=o(T)$. Our algorithms consist of two main components: a static regret minimizer based on Optimistic Mirror Descent with a carefully chosen optimism vector to handle the one-sided Lipschitzness of the reward function, and a restart scheme with adaptive batch sizes to adapt to the unknown $V_T$ without prior knowledge or to achieve an improved dynamic regret rate (in regimes where $L_T=o(T)$). Finally, for a specific auction problem instance, it is not a priori clear which regularity metric on the sequence of the opponents' highest bids leads to a smaller dynamic regret, and we use the meta algorithm by \citet{sani2014exploiting} to establish a best-of-both-worlds dynamic regret guarantee.





\subsection{Dynamic Regret Rates under the Temporal Variation Constraint}
\label{sec:full_vt}
We first focus on the case where $V_T=o(T)$. In Section \ref{sec:full_why}, we provide a step-by-step illustration of why previous approaches do not work in a straightforward adaptation. In Section \ref{sec:optimal}, we discuss our minimax-optimal policy, particularly how to design the optimism and how to restart the algorithm with adaptive batch sizes to eliminate the requirement for knowing $V_T$ in advance.

\subsubsection{Why Existing Works Do Not Directly Apply?}
\label{sec:full_why}
In this section, we discuss several previous approaches that do not work in a straightforward manner for achieving optimal dynamic regret in our setting. These include: (i) the policy proposed by \citet{jadbabaie2015online} for achieving optimal dynamic regret for convex and Lipschitz functions, (ii) restarting the Hedge algorithm with a fixed batch size, and (iii) restarting the policy from \citet{zhang2022leveraging} with a fixed batch size. None of these approaches achieve the optimal dynamic regret rates in our context. Specifically, the approach in \citet{jadbabaie2015online} heavily relies on the Lipschitzness and cannot handle the one-sided Lipschitz reward. The restart Hedge approach fails to adapt to the slowly varying trend of the opponents' highest bid sequence. While the approach in \citet{zhang2022leveraging} possesses some ability to adapt to the opponents' highest bid sequence, it lacks sufficient flexibility to optimally balance the static regret and the transition cost, thus leading to a suboptimal dynamic regret rate.




We begin by briefly reviewing the setting of online convex optimization (OCO), as the policy proposed by \citet{jadbabaie2015online} is developed within this framework. OCO models a sequential decision problem as a $T$ round zero-sum game between a learner and an adversary. At round $t$, the learner chooses $x_t$ from $\cX$, a convex decision set and the adversary reveals $f_t$, a convex loss function. An OCO algorithm $\cA$ (possibly randomized) maps the historical losses to the current decision: $x_t = \cA(f_1, \dots, f_{t-1}) \in \cX$. The static regret of OCO is defined as 
\[
\E[\text{R}_T(\pi)]=\sumt \E[f_t(x_t)]-\min_{x\in\cX}\sumt f_t(x).
\]
We refer to $\min_{x \in \cX} \sumt f_t(x)$ as the \emph{static benchmark} of OCO. Inspired by non-stationary stochastic optimization problems, \citet{besbes2015non} observe that $\sumt \min_{x_t^* \in \cX} f_t(x_t^*)$ (which they term the \emph{dynamic benchmark}) forms a strictly stronger benchmark. The dynamic regret can be defined as:
\[
\E[\text{DR}_T(\pi)]=\sumt \E[f_t(x_t)]-\sumt \min_{x_t^* \in \cX} f_t(x_t^*).
\]

It is well-known \citep{besbes2015non,jadbabaie2015online,yang2016tracking} that the dynamic regret cannot be sublinear in $T$ if the loss functions $f_1, f_2, \dots, f_T$ are chosen arbitrarily. A common assumption, considered by \citet{besbes2015non,jadbabaie2015online}, constrains the temporal variation of the loss sequence to be sublinear in $T$. More precisely, it is assumed that $V_T \coloneqq \sum_{t=2}^T \|f_t - f_{t-1}\|_\infty = o(T)$, where $\|f_t - f_{t-1}\|_\infty \coloneqq \sup_{x \in \mathcal{X}} |f_t(x) - f_{t-1}(x)|$. In the case of exact gradient feedback, an $O(V_T)$ upper bound can be achieved \citep{jadbabaie2015online} by submitting $x_t = \argmin_{x \in \mathcal{X}} f_{t-1}(x)$. With noisy gradients, an $O(T^{2/3}V_T^{1/3})$ bound is achievable by restarting the OGD algorithm with a fixed batch size \citep{besbes2015non}. The dynamic regret guarantees of both policies are minimax-optimal.

Here, we consider a one-sided Lipschitz reward function, which presents a significantly greater challenge than convex loss functions. However, we operate in a noiseless setting where $m_t$ is revealed exactly. This aligns more closely with the setting in \citet{jadbabaie2015online}. Following this line of reasoning, one might consider the bidding strategy $b_t=\argmax_{b\in[0,1]}r(b;v_{t-1},m_{t-1})$. However, the following example illustrates why this approach is insufficient.

\begin{example}
    Suppose the learner bids $b_t = \argmax_{b \in [0,1]} r(b; v_{t-1}, m_{t-1})$ while the adversary chooses $v_t \equiv 1$ and $m_t = \frac{t}{T}$ for $t \in [T]$. Then the learner suffers $\Omega(T)$ dynamic regret.
    This occurs because, with monotonically increasing $m_t$, the bidder consistently underbids and receives zero revenue due to the one-sided Lipschitz property, while the dynamic benchmark bidding $b_t^* = m_t$ wins every auction and accumulates revenue of $1-\frac{t}{T}$ at round $t$.
\end{example}
We now explore more advanced techniques to address this problem. A key challenge in non-stationary online learning is that the reward sequence may exhibit continuous drifts or abrupt shifts, which diminishes the reliability of older data. Consequently, many existing approaches incorporate mechanisms to "forget" old data, either explicitly or implicitly.  In this work, we focus on the restart scheme proposed by \citet{besbes2015non,besbes2019optimal}, partitioning the time horizon $T$ into $n$ batches, denoted by $\mathcal{T}_j$, each of length $\Delta_{T,j}$. While \citet{besbes2015non} consider fixed batch lengths, we allow varying lengths for greater flexibility. Adapting \citet[Proposition 2]{besbes2015non} to our online first-price auction problem, the dynamic regret can be decomposed as follows:
\begin{equation}\label{eq:decomp1}
    \begin{split}
        &\E[\text{DR}_T(\pi)]=\sup_{b_1^*,\dots,b_T^*\in[0,1]}\sumt\lt(r(b_t^*;v_t,m_t)-\E[r(b_t;v_t,m_t)]\rt)\\
=&\sum_{j=1}^n\lt(\max_{f\in\tF}\sumtj r(f(v_t);v_t,m_t)-\sumtj \E[r(b_t;v_t,m_t)]\rt)\\
&+\sum_{j=1}^n\lt(\sumtj r(b_t^*;v_t,m_t)-\max_{f\in\tF}\sumtj r(f(v_t);v_t,m_t)\rt)\\
\coloneqq&\sum_{j=1}^n\cS^{\cA}(\tF,\cT_j)+\sum_{j=1}^n\mathcal{C}(\tF,\cT_j).\\
    \end{split}
\end{equation}
We decompose the dynamic regret over the time horizon $T$ into contributions from $n$ batches. The dynamic regret within each batch $\mathcal{T}_j$ is further decomposed into the static regret and a transition cost. Specifically, $\mathcal{S}^{\mathcal{A}}(\tF, \mathcal{T}_j)$ denotes the static regret of algorithm $\mathcal{A}$ applied to batch $\mathcal{T}_j$ against the best fixed policy in a policy class $\tF$. The term $\mathcal{C}(\tF, \mathcal{T}_j)$ represents the transition cost from static to dynamic regret for batch $\mathcal{T}_j$ and policy set $\tF$. 

To demonstrate the application of the decomposition in Equation \eqref{eq:decomp1} for achieving sublinear dynamic regret, we consider using the restart scheme with the Hedge algorithm as $\mathcal{A}$. We partition the time horizon into batches of equal length $\Delta_T$, with the possible exception of the last batch. Let $\tF$ be the set of constant policies, i.e., $\tF \coloneqq \{f(v; \tau) = \tau | \tau \in [0,1]\}$. Then, the following proposition holds:
\begin{proposition}\label{prop:restart_hedge}
    Assume $V_T=o(T)$ and is known, $V_T^v\coloneqq\sum_{t=2}^T|v_t-v_{t-1}|=o(T)$, then  restarting the Hedge policy every $\Delta_T$ rounds, where $\Delta_T=O\lt(\lt(\frac{T}{V_T+V_T^v}\rt)^{\frac{2}{3}}\rt)$ achieves $\tO\lt(T^{\frac{2}{3}}(V_T+V_T^v)^{\frac{1}{3}}\rt)$ dynamic regret.
\end{proposition}
Finally, we examine the results of \citet{zhang2022leveraging}, which study online first-price auctions where a hint $h_t$ is provided before deciding the bid $b_t$. The hint satisfies $\E\lt[|h_t-m_t|^q\rt]\le\sigma_t^q$
for any $t\in[T]$. The \emph{single hint setting} is considered in their work, where they assume an upper bound on $\sum_{t=1}^T \sigma_t$ is available. Our problem can be viewed as a special case of the single hint setting by choosing $h_t = m_{t-1}$, $q = 1$, and $V_T = \sum_{t=1}^T \sigma_t$. Then, \citet[Theorem 2]{zhang2022leveraging} demonstrate that, when $v_t \equiv 1$, there exists an algorithm that guarantees the following static regret bound:
\begin{equation}\label{eq:zhang}
\E[\text{R}_T(\pi)] = \max_{f \in \cFl} \sum_{t=1}^T r(f(v_t); v_t, m_t) - \sum_{t=1}^T \E\left[r(b_t; v_t, m_t)\right] = \tO\left(T^{\frac{1}{4}} V_T^{\frac{1}{4}}\right),
\end{equation}
where $\cFl$ is the set of $1$-Lipschitz policies $f: [0, 1] \to [0, 1]$.

By combining the restart scheme with the algorithm in \citet[Theorem 2]{zhang2022leveraging}, we obtain the following result. 

\begin{proposition}\label{prop:restart_zhang}
    Assume $v_t\equiv 1$ holds for $t\in[T]$ and $V_T=o(T)$ is known, then the learner can restart the algorithm in \citet[Theorem 2]{zhang2022leveraging} every $\Delta_T$ rounds, where $\Delta_T=O\lt(\lt(\frac{T}{V_T}\rt)^{\frac{1}{2}}\rt)$ to achieve $\tO\lt(\sqrt{TV_T}\rt)$ dynamic regret.
\end{proposition}
    
Proposition \ref{prop:restart_hedge} establishes an $\tilde{O}(T^{2/3}(V_T + V_T^v)^{1/3})$ upper bound on the dynamic regret. However, this bound is suboptimal compared to the $\Omega(\sqrt{TV_T})$ lower bound presented in Theorem \ref{thm:full_lb}. While Proposition \ref{prop:restart_zhang} achieves the optimal rate, it relies on the restrictive assumption that $v_t \equiv 1$ for $t\in[T]$. Although \citet{zhang2022leveraging} consider varying $v_t$ as well, they employ the ChEW policy \citep{han2020learning}, which is an inefficient variant of the SEW policy (also from \citet{han2020learning}), to achieve $\tilde{O}(\sqrt{T})$ static regret. Directly combining this rate with the restart scheme and following the proof of Proposition \ref{prop:restart_zhang} results in a dynamic regret of $\tO\lt(T^{\frac{2}{3}}V_T^{\frac{1}{3}}\rt)$, which is still suboptimal.

Consequently, achieving the optimal dynamic regret rate for varying $v_t$ using existing approaches remains an open problem. We will explore alternative methodologies to address this.




\subsubsection{Minimax-Optimal Policy and Parameter-free Scheme.}
\label{sec:optimal}
In this section, we investigate how to achieve the minimax-optimal dynamic regret upper bound. 
Our main approach is to design a suitable restart scheme that employs the framework of Optimistic Mirror Descent (Algorithm \ref{alg:omd}, Option \rtwo) as the static regret minimizer. The key technical contribution is to provide a novel optimism configuration $\mu_t=\max\{v_t-m_t,0\}$, which yields a favorable balance between the static regret and the transition cost, thus leading to the optimal dynamic regret rate. When $V_T$ is known, a constant batch size suffices to achieve the optimal dynamic regret. When $V_T$ is unknown, we can employ an adaptive batch size (Algorithm \ref{alg:full_vt}) to achieve the optimal dynamic regret as well. We provide insights and details about the optimal policy below.

Recall that the proof in Proposition \ref{prop:restart_hedge} is based on the following argument:
\begin{equation} \label{eq:dynamic_regret_initial}
    \begin{split}
        \E[\text{DR}_{T}(\pi)] &=\sum_{j=1}^n\cS^{\cA}(\tF,\cT_j)+\sum_{j=1}^n\mathcal{C}(\tF,\cT_j)\\
        &\le \left\lceil \frac{T}{\Delta_T} \right\rceil \cdot \tilde{O}\left(\sqrt{\Delta_T}\right) + \Delta_T (V_T+V_T^v) \\
        &= \tilde{O}\left(\frac{T}{\sqrt{\Delta_T}} + \Delta_T (V_T+V_T^v)\right) = \tilde{O}\left(T^{\frac{2}{3}} (V_T+V_T^v)^{\frac{1}{3}}\right)
    \end{split}
\end{equation}
with optimal tuning of the batch size $\Delta_T$. The $O(\sqrt{\Delta_T})$ static regret achieved with this tuning, while minimax-optimal for each batch $j \in [n]$, is not tight when batch $j$'s temporal variation, $V_{T,j} = \sum_{t \in \cT_j} |m_t - m_{t-1}|$, is significantly smaller than $\Delta_T$. For instance, if the values $m_t$ are constant within batch $\mathcal{T}_j$, we expect $O(1)$ static regret rather than the minimax-optimal $O(\sqrt{\Delta_T})$. This observation leads us to investigate the existence of online learning policies with static regret bounds that scale with the temporal variation of the opponents' highest bid sequence $(m_t)_{t=1}^T$. In the machine learning theory community, this question aligns with the concept of \emph{adaptive online learning} \citep{cesa2007improved,rakhlin2013optimization,wei2018more}, which focuses on achieving static regret guarantees that scale with the "complexity" of the input data.

Inspired by this observation and the $\Omega(\sqrt{TV_T})$ lower bound that we will establish in Section \ref{sec:full_vt_lb}, we conjecture that an improved dynamic regret bound can be achieved by considering:
\begin{equation} \label{eq:dynamic_regret_improved}
\begin{split}
\E[\text{DR}_{T}(\pi)] &=\sum_{j=1}^n\cS^{\cA'}(\cF',\cT_j)+\sum_{j=1}^n\mathcal{C}(\cF',\cT_j)\\
&\stackrel{?}{\le} \sum_{j=1}^{\lceil{T/\Delta_T}\rceil} \tilde{O}\left(\Delta_T V_{T,j} + 1\right) + \Delta_T V_T \\
&= \tilde{O}\left(\Delta_T V_T+\frac{T}{\Delta_T}\right)+ \Delta_T V_T = \tilde{O}\left(\Delta_T V_T+\frac{T}{\Delta_T}\right) = \tilde{O}\left(\sqrt{T V_T}\right),
\end{split}
\end{equation}
where we replace the minimax-optimal policy $\mathcal{A}$ and class $\tF$ with a potentially different policy $\mathcal{A}'$ and class $\mathcal{F}'$, aiming for a regret guarantee that scales with the intra-batch temporal variation. The step marked with $\stackrel{?}{\le}$ is the crux of our approach and requires establishing that the static regret can indeed scale with $V_{T,j}$ within each batch (to be elaborated in the sequel).

While it may initially seem surprising that such an adaptive policy could improve dynamic regret, given that $\Delta_T V_{T,j}$ can exceed $\sqrt{\Delta_T}$ for some $j \in [n]$, the adaptive nature of $\mathcal{A}'$ and the fact that $\sum_{j=1}^n V_{T,j} \le V_T$ allow for a more favorable balance between the overall static regret $\sum_{j=1}^n \mathcal{S}^{\mathcal{A}'}(\mathcal{F}', \mathcal{T}_j)$ and the overall transition cost $\sum_{j=1}^n \mathcal{C}(\mathcal{F}', \mathcal{T}_j)$. This permits a more aggressive choice of $\Delta_T$, leading to an improved dynamic regret rate. We refer to this idea as "adaptive balancing," as it leverages adaptive online learning algorithms to balance the scales of the static regret and the transition cost.

To achieve an $\tilde{O}(\Delta_T V_{T,j} + 1)$ static regret bound, we require an algorithm satisfying two conditions: (i) its regret should scale with the temporal variation of the sequence $(m_t)_{t\in\cT_j}$, 
and (ii) it should be customizable to facilitate adaptive balancing. The OMD framework \citep{chiang2012online,rakhlin2013optimization} fulfills both requirements. In particular, we focus on Option \rtwo of Algorithm \ref{alg:omd}, which implies a static regret bound of the form $O\lt(\sqrt{\sum_{t=1}^T (r_{t,i^*} - \mu_t)^2 \ln N}\rt)$ \citep{steinhardt2014adaptivity,wei2018more}, where $i^*$ is the index of the optimal expert in hindsight. As mentioned in Section \ref{sec:omd}, the optimism vector $\mu_t\cdot\mathbf{1}$ plays an important role in balancing the static regret and the transition cost. We choose
\begin{equation}\label{eq:opt}
    \mu_t = \max\{v_t - m_t, 0\},
\end{equation}
which is a novel contribution of this work. We provide intuition about how we derive this optimism below. Notably, this choice of $\mu_t$ coincides with $r(b_t^*; v_t, m_t)$ in Equation \eqref{eq:dynamic}, which helps to relate the static regret and the transition cost—a point that will be more transparent in the proof of Theorem \ref{thm:full_vt_ub}.

As discussed in Section \ref{sec:omd}, when the optimism vector is a constant times an all-ones vector, such as $\mu_t\cdot\mathbf{1}$, $\mu_t$ can depend on $r_t$, the reward at round $t$. Since the reward $r_t$ depends on both the private valuation $v_t$ and the opponents' highest bid $m_t$, it is natural to parametrize $\mu_t$ as a function of these two variables. We assume $\mu_t=\mu(v_t,m_t)$, which turns out to make our theory work after some calculations. Next, we discuss how to determine the optimism $\mu_t=\mu(v_t,m_t)$.


 When we restrict our focus to $\mathcal{T}_j$, the $j$-th batch, we need an algorithm with regret upper bounded by $\tilde{O}(1+\Delta_T V_{T,j})$, as illustrated in Equation \eqref{eq:dynamic_regret_improved}. We analyze the problem instance in Example \ref{ex:optimism}, which contains several parameters like $v,m$ and $\tilde{m}$. By examining different regimes of these parameters, we find that choosing $\mu_t$ as in Equation \eqref{eq:opt} is indeed reasonable. This optimism can be combined with the fact that $i^*$ is the optimal expert to show the desired adaptive static regret bound.
While here we gain insights using special examples, later we find this optimism indeed works in general. Therefore, we can achieve $\tilde{O}(1+\Delta_T V_{T,j})$ static regret by combining Option \rtwo of Algorithm \ref{alg:omd} with Equation \eqref{eq:opt}.

\begin{example}\label{ex:optimism}
    Consider the following first-price auction instance on batch $\cT_j$, $v_t\equiv v$ for $t\in\mathcal{T}_j$ and 
\[
(m_t)_{t\in\mathcal{T}_j}=(\underbrace{m,m,\dots,m}_{T_1 \text{ copies }}, \underbrace{\tilde{m},\tilde{m},\dots,\tilde{m}}_{T_2 \text{ copies}}),
\]
where $T_1+T_2=\Delta_T$, the batch size.
\end{example}

However, computing $p_{t+1}'$ in Algorithm \ref{alg:omd} with Option \rtwo requires solving a convex optimization problem, which can be computationally expensive. Therefore, we employ the Prod forecaster \citep{cesa2007improved}, which offers the same $O(\sqrt{\sum_{t=1}^T (r_{t,i^*} - \mu_t)^2 \ln N})$ regret guarantee with more efficient updates:
    \begin{equation} \label{eq:prod_forecaster}
    p_1 = \left(\frac{1}{N},\dots,\frac{1}{N}\right), \qquad p_{t+1,i} = \frac{(1 + \eta (r_{t,i} - \mu_t)) p_{t,i}}{\sum_{j=1}^N (1 + \eta (r_{t,j} - \mu_t)) p_{t,j}},
    \end{equation}

Furthermore, the dynamic regret bound in Proposition \ref{prop:restart_hedge} has an undesirable dependence on $V_T^v\coloneqq\sum_{t=2}^T|v_t-v_{t-1}|$. We aim to eliminate this dependence, which arises from the one-sided Lipschitz property of the reward function:
    \begin{lemma}\label{lem:lipsv1}\citep{han2020learning}
        For any $v,m\in[0,1]$, $b\leq \min\{v,b'\}$,
        \[
        r(b;v,m)-r(b';v,m)\leq b'-b.
        \]
    \end{lemma}

     Lemma \ref{lem:lipsv1} implies that the one-sided Lipschitzness of the reward function relies on the condition $b \leq v$, meaning that the set of constant policies does not satisfy this property. Notably, \citet{han2020learning} encountered a similar difficulty, where they aimed to compete with the best fixed policy within the set of $1$-Lipschitz policies $\cF_{\text{Lip}}$. However, they found that restricting the policy set to $\cF_0 \coloneqq \{f  \mid f \in \cF_{\text{Lip}}, f(v) \leq v\}$ does not compromise the reward and resolves the problem. Inspired by this, we define $\mathcal{F} \coloneqq \{f(v; \tau) \mid \tau \in [0, 1]\}$, where $f(v; \tau) \coloneqq \min\{v, \tau\}$. $\mathcal{F}$ can be viewed as a modified version of $\mathcal{N} \coloneqq \{\tau \mid \tau \in [0, 1]\}$, the set of constant policies, with the additional constraint $f(v; \tau) \leq v$.
 We further define $\mathcal{F}_\epsilon \coloneqq \{f(v; \tau) \mid \tau \in \{0, \epsilon, 2\epsilon, \dots, \epsilon \lfloor 1/\epsilon \rfloor\}\}$, which is a discretized version of $\mathcal{F}$ with precision $\epsilon$. Using this setup, we can effectively eliminate the
dependence on $V_T^v$ through a careful application of the one-sided Lipschitzness property given in Lemma \ref{lem:lipsv1}.

With all the necessary tools in place, we now illustrate how to leverage the concept of "adaptive balancing" to achieve an improved dynamic regret rate. Assuming $V_T$ is known, it is sufficient to restart the Prod forecaster every $\Delta_T$ rounds, where $\Delta_T = O\lt(\sqrt{\frac{T}{V_T}}\rt)$, to achieve a dynamic regret of $\tO\lt(\sqrt{TV_T}\rt)$. However, in practice, $V_T$ is typically unknown. To address this, we use an adaptive restart condition, as demonstrated in Algorithm \ref{alg:full_vt}, to resolve the issue while still achieving the minimax-optimal rate. Theorem \ref{thm:full_vt_ub} establishes the minimax-optimal dynamic regret guarantee under the condition $V_T=o(T)$.  

\begin{algorithm}[H]
\caption{The Adaptive Restart Prod Policy (AR-Prod)}
    \label{alg:full_vt}
    \SetAlgoLined
    \SetKwInOut{Input}{Input}
    \Input{Time horizon $T$}
    \BlankLine
    $j\leftarrow 1,\eta\leftarrow \frac{1}{2},\epsilon\leftarrow \frac{1}{T},c\leftarrow \frac{1}{T}$\; 
    \While{$t\le T$}{
        Observe the ad impression at $t$ and generate the value $v_t$\;
        Create $\cT_j$\;
        $++j$\;
        $p_t \leftarrow  \left(\frac{1}{N},\dots,\frac{1}{N}\right)$ where $N\leftarrow \frac{1}{\eps}$\;
        \While{$\Delta_{T,j}<\sqrt{\frac{T}{\sum_{i= 1}^jV_{T,i}+c}}$}{
        Choose $b_t\leftarrow \min\{v_t,i\eps\}$ with probability $p_{t,i}$\;
        Submit $b_t$ and receive $m_t$\;
        Update $\Delta_{T,j}$ and $V_{T,j}$\tcp*[r]{$\Delta_{T,j}$: length of $\mathcal{T}_j$, $V_{T,j}$: temporal variation of $(m_t)_{t\in\cT_j}$}
        $\mu_t\leftarrow\max\{v_t-m_t,0\}$\;
            $p_{t+1,i} \leftarrow  \frac{(1 + \eta (r_{t,i} - \mu_t)) p_{t,i}}{\sum_{j=1}^N (1 + \eta (r_{t,j} - \mu_t)) p_{t,j}}$\;
            ++t\;
            Observe the ad impression at $t$ and generate the value $v_t$\;
        }
    }
    \end{algorithm}


\begin{theorem}\label{thm:full_vt_ub}
    Assume $V_T=o(T)$. When $V_T$ is known, we can restart the Prod forecaster (Equation \eqref{eq:prod_forecaster}) with $\mu_t=\max\{v_t-m_t,0\}$ using a constant batch size $\Delta_T=O\left(\sqrt{\frac{T}{V_T}}\right)$ to achieve the $\tO\left(\sqrt{TV_T}\right)$ dynamic regret. When $V_T$ is unknown, Algorithm \ref{alg:full_vt} restarts the Prod policy adaptively and achieves
    \[
    \sup_{(v_t,m_t)_{t=1}^T\in\cV}\E\lt[\textnormal{DR}_T(\pi)\rt]=\tO\lt(\max\lt\{\sqrt{TV_T},1\rt\}\rt).
    \]
\end{theorem}
\begin{proof}{Proof Sketch.}
    We begin by considering the case where $V_T$ is known. The proof follows the approach suggested in Equation \eqref{eq:dynamic_regret_improved}. We define $\mathcal{F} \coloneqq \{f(v;\tau) \mid \tau \in [0, 1]\}$ and $\mathcal{F}_\eps \coloneqq \lt\{f(v;\tau) \mid \tau =k\eps,k\in\lt[\floor*{\frac{1}{\eps}}\rt]\rt\}$ as the set of policies and its discretization, respectively, where $f(v;\tau)\coloneqq\min\{v,\tau\}$. We first consider the case where $V_T$ is known. We divide the time horizon $T$ into batches $\cT_1,\cT_2,\dots,\cT_n$ of equal length (possibly except $\cT_n$) and consider the dynamic regret
    \begin{equation}\label{eq:dynamic_minimax}
        \begin{split}
        \E[\text{DR}_{T}(\pi)] &=\sum_{j=1}^n\cS^{\cA}(\cF,\cT_j)+\sum_{j=1}^n\mathcal{C}(\cF,\cT_j)\\
        &\le \sum_{j=1}^{\lceil{T/\Delta_T}\rceil} \tilde{O}\left(\Delta_T V_{T,j} + 1\right) + \sum_{j=1}^{\lceil{T/\Delta_T}\rceil}\Delta_T V_{T,j} \\
        &= \tilde{O}\left(\Delta_T V_T+\frac{T}{\Delta_T}\right)+ \Delta_T V_T= \tilde{O}\left(\Delta_T V_T+\frac{T}{\Delta_T}\right) = \tilde{O}\left(\sqrt{T V_T}\right),
        \end{split}
        \end{equation}
        where $\cA$ is the Prod forecaster illustrated in Equation \eqref{eq:prod_forecaster}. The inequality is shown by the following idea: by choosing the translation term in the Prod forecaster as $\mu_t\coloneqq \max\{v_t-m_t,0\}$, we can show 
        \begin{equation}\label{eq:relate}
            S^{\cA}(\cF,\cT_j)=O(\mathcal{C}(\cF,\cT_j))+\tO(1)
        \end{equation}
        holds for any batch $j$. It then suffices to show $\mathcal{C}(\cF,\cT_j)\le\Delta_T V_{T,j}$ to establish that the inequality in Equation \eqref{eq:relate} holds, which is possible since our expert set $\cF$ facilitates the application of Lemma \ref{lem:lipsv1}.

        

    Now suppose $V_T$ is unknown, then we use the adaptive restart routine in Algorithm \ref{alg:full_vt}. Following the argument in Equation \eqref{eq:dynamic_minimax}, we can establish
    \begin{equation}\label{eq:decomp_sketch4}
        \E[\text{DR}_T(\pi)]=\tO\lt(n+\sum_{j=1}^n\Delta_{T,j}V_{T,j}\rt),
    \end{equation}
    where $n$ denotes the number of batches, $\Delta_{T,j}$ represents the length of batch $j$, and $V_{T,j}$ denotes the temporal variation of $m_t$ within batch $j$. While these quantities ($n$, $\Delta_{T,j}$, and $V_{T,j}$) are a priori unknown, leveraging the restart condition in conjunction with the self-confident tuning technique (cf. \cite{auer2002adaptive}) allows us to effectively bound them. Specifically, these techniques yield $n = \tilde{O}(\sqrt{TV_T})$ and $\sum_{j=1}^n \Delta_{T,j} V_{T,j} = \tilde{O}(\sqrt{TV_T})$, where $T$ is the total time horizon and $V_T$ denotes the total temporal variation across all batches. Consequently, substituting these bounds into Equation \eqref{eq:decomp_sketch4} yields the desired $\tilde{O}(\sqrt{TV_T})$ bound. 
\end{proof}
\begin{remark}
    Previous proofs for learning in non-stationary environments \citep{besbes2015non,besbes2019optimal,cheung2022hedging,cheung2023nonstationary} typically decompose the dynamic regret into the sum of static regret and transition cost, and then bound these terms \emph{individually}. While this approach could also be applied to our problem, in the proof of Theorem \ref{thm:full_vt_ub}, we instead establish a direct relationship between the static regret and the transition cost (Equation \eqref{eq:relate}). This alternative approach results in a more transparent proof and may be of independent interest.
\end{remark}

\subsection{Dynamic Regret Rates under the Switching Number Constraint}
\label{sec:full_lt}
We now consider the case where the number of switches in the opponents' highest bid sequence, $L_T = \sum_{t=2}^T \ind(m_t \neq m_{t-1})$, is $o(T)$. Our approach combines the Optimistic Mirror Descent (OMD) framework (Algorithm \ref{alg:omd}, Option \rone) with an adaptive restart scheme: OMD with a suitable optimism vector $o_t$ is run within each batch, and a new batch is started whenever a change in $m_t$ is detected (i.e., $m_t \neq m_{t-1}$).

    Since $m_t$ is observed exactly, each batch contains at most one switch ($m_t\ne m_{t-1}$). Due to the configured optimism, the static regret for each batch corresponds to the number of switches. Given this single-switch property, we can show the transition cost from static regret to dynamic regret is $\tO(1)$. Combining both parts, and summing over all $L_T$ batches, the total dynamic regret is $\tO(L_T)$. We use the negative entropy regularizer in OMD, which allows for efficient closed-form updates as shown in Algorithm \ref{alg:full_lt}. Theorem \ref{thm:full_lt1} formalizes this result, establishing a dynamic regret upper bound. 

 \begin{algorithm}[H]
\caption{Adaptive Restart Optimistic Mirror Descent (AR-OMD)}
    \label{alg:full_lt}
    \SetAlgoLined
    \SetKwInOut{Input}{Input}
    \Input{$\cP$ is the convex hull of $\{e_1,\dots,e_N\}$; $\psi(p)\leftarrow \frac{1}{\eta}\sum_{i=1}^Np_i\ln p_i$}
    $j\leftarrow 1$,$t\leftarrow 1$\;
    \BlankLine
    \While{$t\le T$}{
        Create $\cT_j$\;
        ++ $j$\;
        Update
        \[
        p_{t,i}\propto\exp\lt(\eta\lt(\sum_{s=1}^{t-1}r_{s,i}+o_{t,i}\rt)\rt)
        \]
        where $o_{t,i}\coloneqq r(f(v_t;i\epsilon);v_t,m_{t-1})$\;
        Submit bids according to $p_t$, and receive $m_t$\;
        \While{$t$ is the first round in $\cT_j$ or $m_t=m_{t-1}$}{
            $++t$\;
            Update
        \[
        p_{t,i}\propto\exp\lt(\eta\lt(\sum_{s=1}^{t-1}r_{s,i}+o_{t,i}\rt)\rt)
        \]
        where $o_{t,i}\coloneqq r(f(v_t;i\epsilon);v_t,m_{t-1})$\;
        }
    }
    \BlankLine
    \end{algorithm}

\begin{theorem}\label{thm:full_lt1}
    Assume $L_T=o(T)$ and is unknown, then Algorithm \ref{alg:full_lt} achieves
    \[
    \sup_{(v_t,m_t)_{t=1}^T\in\cL}\E\lt[\textnormal{DR}_T(\pi)\rt]=\tO\lt(L_T\rt).
    \]
\end{theorem}

\subsection{Best-of-Both-Worlds Dynamic Regret}
\label{sec:bobw}
In Sections \ref{sec:full_vt} and \ref{sec:full_lt}, 
we have established the $\tilde{O}(\sqrt{V_TT})$ and $\tilde{O}(L_T)$ dynamic regret rates for slowly varying and abruptly changing bidding environments, respectively. But in reality, it is hard for a learner to know a priori which non-stationary measure is suitable, thus it is desirable to automatically achieve the better of the two guarantees whenever one outperforms the other. This problem is termed as the \emph{best-of-both-worlds} bound in the online learning literature. An important technique for establishing the best-of-both-worlds bound is to run a few base algorithms in parallel, and use a meta algorithm to aggregate the output of base algorithms to ensure the resulting long-term performance is always as good as the best base algorithm. In this part, we establish the best-of-both-worlds bound based on the meta algorithm by \citet{sani2014exploiting}. The resulting algorithm and theoretical guarantee are presented as Algorithm \ref{alg:best} and Theorem \ref{thm:bobw}, respectively, and the proof of Theorem \ref{thm:bobw} is deferred to Appendix \ref{sec:bobw}. 

  \begin{theorem}\label{thm:bobw}
        Assume $V_T=\Omega(\ln T)$, then Algorithm \ref{alg:best} achieves $\tO(\min\{\sqrt{TV_T},L_T\})$ best-of-both-worlds dynamic regret guarantee for online first-price auctions.
    \end{theorem}

\begin{algorithm}[H]
    \caption{Non-stationary First-price Auction with Best-of-Both-Worlds Guarantee.}
    \SetAlgoLined
    \SetKwInOut{Input}{Input}
    \Input{Let $\cA$ and $\cB$ be Algorithms  
    \ref{alg:full_vt} and \ref{alg:full_lt}, respectively; total number of rounds $T$, learning rate $\eta=\frac{1}{2}\cdot\sqrt{\frac{\ln T}{T}}$, initial weights $w_1^{\cA}=\eta,w_1^{\cB}=1-\eta$.}
    $j\leftarrow 1$,$t\leftarrow 1$\;
    \BlankLine
    \For{$t \leftarrow 1$ \KwTo $T$}{
        $p_t=\frac{w_t^{\cA}}{w_t^{\cA}+w_t^{\cB}}$\;
        Observe bids $b_t^{\cA}$ and $b_t^{\cB}$ produced by $\cA$ and $\cB$\;
        {%
            Bid
            \begin{flushleft}
        \[
            b_t=\begin{cases}
                b_t^{\cA},&\text{with probability } p_t,\\
                b_t^{\cB},&\text{otherwise},
            \end{cases}
        \]
        \end{flushleft}}%
        Observe $m_t$ and get reward $r(b_t;v_t,m_t)$\;
        Send $m_t$ to $\cA$ and $\cB$\;
        Let $\delta_t=r_t(b_t^{\cA};v_t,m_t)-r_t(b_t^{\cB};v_t,m_t)$\;
        Set $w_{t+1}^{\cA}=w_t^{\cA}(1+\eta\delta_t)$\;
}
    \BlankLine
    \label{alg:best}
    \end{algorithm}
  
    \begin{remark}
        If we want to obtain a best-of-both-worlds static regret bound for two algorithms with different adaptive static regret guarantees, the meta algorithm by \citet{sani2014exploiting} might not be applicable since the overhead can be as large as $O(\sqrt{T\ln T})$ for one adaptive regret guarantee, while it is $O(1)$ for the other adaptive regret guarantee. The $O(\sqrt{T\ln T})$ overhead can destroy an adaptive regret guarantee. Fortunately, the dynamic regret guarantee in Theorem \ref{thm:full_vt_ub} is $\tO(\sqrt{TV_T})$, which can easily absorb the $O(\sqrt{T\ln T})$ overhead as long as $V_T=\Omega(\ln T)$.
    \end{remark}

\section{Dynamic Regret Lower Bounds}
\label{sec:full_lb}
In this section, we demonstrate how to establish minimax lower bounds for the class of auction instances where the opponents' highest bid sequence is constrained by either $V_T$ or $L_T$. Our main effort focuses on the case of $V_T$. Following \citet{besbes2015non}, we term $V_T$ the \emph{variation budget}. For the corresponding lower bound construction, we partition the time horizon $T$ into batches of equal size $H$, and we allocate a small and fixed amount of variation budget $\frac{1}{H}$ to each batch to create a jump at locations drawn from the uniform distribution. Within each batch, due to the one-sided Lipschitzness of the reward function, the learner faces a dilemma: providing a small bid incurs $0$ dynamic regret before the jump occurs, but will incur $\Omega(1)$ dynamic regret at the jump point. Bidding a higher price avoids the $\Omega(1)$ dynamic regret at the jump point, but incurs $\frac{1}{H}$ dynamic regret for each round until the jump point. Formally, we show that any non-anticipatory policy incurs $\Omega(1)$ dynamic regret within each batch based on dynamic programming in Lemma \ref{lem:lb}. Since there are $\Theta(T/H)$ batches, choosing $H=\Theta(\sqrt{T/V_T})$ satisfies the variation budget constraint and also implies the $\Omega(\sqrt{TV_T})$ dynamic regret lower bound. For the case of $L_T$, we achieve the $\Omega(L_T)$ lower bound by reducing to the case of $V_T$, which is possible because each batch considered in Lemma \ref{lem:lb} contains only one random jump.



\subsection{Minimax Lower Bound under the Temporal Variation Constraint.}
\label{sec:full_vt_lb}

In this section, we establish an $\Omega(\sqrt{TV_T})$ lower bound for online first-price auctions. We begin by outlining the technical challenges. The adversary's objective is to optimally allocate the variation budget across the entire time horizon. Existing lower bounds for learning in non-stationary environments \citep{besbes2015non,besbes2019optimal,cheung2022hedging} typically rely on the presence of noisy feedback. This noise allows the construction of two reward functions and the partitioning of the time horizon into batches of size $\Delta_T$. Within each batch, a reward function is selected uniformly at random and applied consistently. The noisy feedback ensures that, within each batch, the learner perceives i.i.d. rewards. Consequently, information-theoretic tools can be employed to lower bound the probability of identifying the true underlying reward function. However, in our setting, the learner observes $m_t$ directly, without noise. This absence of noise necessitates the development of alternative approaches.

In this work, for the lower bound construction, we design problem instances that satisfy the variation budget constraint, where the dynamic regret of any non-anticipatory policy can be computed using dynamic programming. Similar ideas have been used to derive minimax lower bounds on the static regret for learning with a small number of experts \citep{cover1966behavior,gravin2016towards,harvey2023optimal}. Specifically, our approach constructs an opponents' highest bid sequence of length $H$ with temporal variation bounded by $1/H$, showing that any admissible policy incurs $\Omega(1)$ dynamic regret on this sequence. By concatenating $\Theta(T/H)$ such sequences with $H = \Theta\lt(\sqrt{T/V_T}\rt)$, we construct a total sequence with temporal variation bounded by $V_T$. The total dynamic regret is then lower bounded by the number of sequences multiplied by $\Omega(1)$, yielding $\Omega(T/H) = \Omega(\sqrt{TV_T})$, as desired.

Lemma \ref{lem:lb} provides the construction of a single sequence and establishes the $\Omega(1)$ lower bound on its dynamic regret. 


\begin{lemma}\label{lem:lb}
    Let $H\ge 2$ be an integer, and consider  a $H$-round online first-price auction game, assume $v_t\equiv 1$, and 
    \[
    m_t=\begin{cases}
        0,&t<\tau,\\
        \delta,&\tau\le t\le H,
    \end{cases}
    \]
    where $\tau$ is uniformly drawn from $\{1,2,\dots,H\}$. Then any non-anticipatory policy suffers at least $\frac{1}{2}-\frac{1}{2H}$ dynamic regret when $\delta=\frac{1}{H}$.
\end{lemma}
Lemma \ref{lem:lb} shows that within a batch of length $H$, a variation budget of $1/H$ can induce $\Omega(1)$ dynamic regret for any admissible algorithm. With a total variation budget of $V_T$, we can construct $\Theta(HV_T)$ such batches. To achieve the desired lower bound, we set $H = \sqrt{T/V_T}$. However, directly concatenating $T/H$ batches as described in Lemma \ref{lem:lb} would result in $m_t$ reaching $\frac{1}{H} \cdot \frac{T}{H} = V_T$ at later stages, potentially violating the assumption that $m_t \in [0, 1]$.

To address this issue, we employ an alternating batch construction. We divide the time horizon into batches of length $H$ and indexed by $j$. For odd $j$, we use the batches constructed in Lemma \ref{lem:lb}. For even $j$, we use batches defined as follows:
\[
    m_t=\begin{cases}
        \delta,&t<\tau,\\
        0,&\tau\le t\le H,
    \end{cases}
    \]
    where $\tau$ is drawn uniformly from ${1, 2, \dots, H}$ and $\delta=\frac{1}{H}$. This alternating construction ensures that $m_t \in [0, 1]$, as depicted in Figure \ref{fig:lb}. By alternating between these types of batches, we can fully utilize the variation budget while respecting the constraint that $m_t\in[0,1]$.

     \begin{figure}[htbp]
    \centering
    \caption{An illustration of the construction for the lower bound. In odd-numbered batches, $m_t$ jumps from $0$ to $\frac{1}{H}$ at random locations, while in even-numbered batches, $m_t$ jumps from $\frac{1}{H}$ back to $0$. The parameter $H$ is carefully chosen to satisfy the variation budget constraint and ensure that we obtain the desired lower bound.}
    \includegraphics[scale=.51]{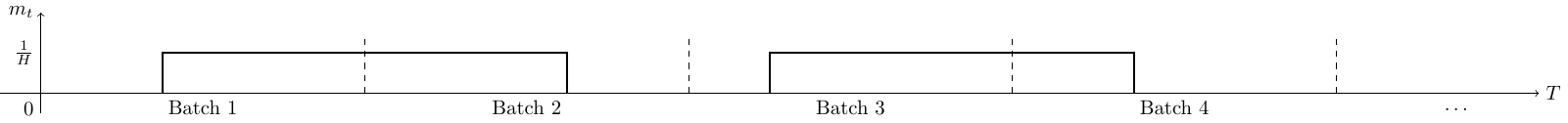}
    \label{fig:lb}
\end{figure}
   



We stitch batches constructed in Lemma \ref{lem:lb}, and establish the minimax-optimal lower bound as follows:
\begin{theorem}\label{thm:full_lb}
    In online first-price auctions, for any $V_T\in\lt[\frac{36}{T},\frac{T}{4}\rt]$, there exists  $(v_t,m_t)_{t=1}^T$ such that $\sum_{t=2}^T|m_t-m_{t-1}|\le V_T$ and the expected dynamic regret of any admissible policy satisfies:
    \[
    \inf_{\pi\in\Pi}\sup_{(v_t,m_t)_{t=1}^T\in\cV}\E\lt[\textnormal{DR}_T(\pi)\rt]\ge\frac{\sqrt{TV_T}}{16},
    \]
    where $\Pi$ is the set of admissible policies.
\end{theorem}

\begin{remark}
    Due to the construction of the lower bound, we can explicitly inform the learner about the creation of the opponents' highest bid batches, the variation budget allocated to each batch, and the total number of batches. The lower bound remains valid under this setting. This implies that our lower bound holds even when the learner is aware of $V_T$, whereas our upper bound does not require prior knowledge of $V_T$.
\end{remark}

\subsection{Minimax Lower Bound under the Discrete Switching Constraint.}
\label{sec:full_lt_lb}

We also establish a corresponding minimax lower bound for the case of $L_T = o(T)$ by reducing it to the proof of Theorem \ref{thm:full_lb}. 
\begin{theorem}\label{thm:full_lt2}
    In online first-price auctions, for any $L_T\in[T]$ and $L_T\le\frac{T}{3}$, there exists $(v_t,m_t)_{t=1}^T$ such that $\sum_{t=2}^T\ind(m_t\ne m_{t-1})\le L_T$ and the expected dynamic regret of any admissible policy satisfies:
    \[
    \inf_{\pi\in\Pi}\sup_{(v_t,m_t)_{t=1}^T\in\cL}\E\lt[\textnormal{DR}_T(\pi)\rt]\ge\frac{L_T}{8}.
    \]
\end{theorem}
It is insightful to compare Theorems \ref{thm:full_lb} and \ref{thm:full_lt2} with Proposition \ref{prop:lb1}. Proposition \ref{prop:lb1} essentially establishes an $\Omega(V_T)$ lower bound for $V_T\in\lt[0,\frac{T}{2}\rt]$ and an $\Omega(L_T)$ lower bound when $L_T = \Theta(T)$. In contrast, the lower bounds in Theorems \ref{thm:full_lb} and \ref{thm:full_lt2} are significantly stronger.

\section{Numerical Experiments}
\label{sec:experiment}
In this section, we conduct numerical experiments to evaluate the performance of our proposed algorithms and compare them with baseline methods. Our experiments consist of two main parts: in the first part, we generate the sequence of opponents' highest bids based on four slowly varying patterns as considered in \citet{besbes2015non,besbes2019optimal,cheung2022hedging}, and then confirm our theoretical findings by evaluating the slopes of the log-log plots of dynamic regret with respect to different time horizons. In the second part, we run our proposed algorithms as well as two baseline methods to bid in a multi-agent bidding environment, where each opponent applies the budget-pacing policy from \citet{gaitonde2022budget}. We find that our algorithms outperform the baselines, especially in regimes where opponents have limited budgets, even when the sequence of opponents' values varies rapidly.

\subsection{Dynamic Regret Growth with Varying Time Horizons}
\label{sec:loglog}

We consider four different patterns for the opponents' highest bid sequence $(m_t)_{t=1}^T$ and then evaluate the slopes of the log-log plots of dynamic regret with respect to different time horizons to confirm the rates predicted by our theoretical findings. The first three patterns are constructed using the following building blocks from \citet{besbes2015non}:
\begin{equation}\label{eq:blocks}
    m_t^{\text{constant}} = \begin{cases}
    0, &  t \le \tau, \\
    1, & t>\tau
\end{cases} \quad
m_t^{\text{exponential}} = \begin{cases}
    0, &  t \le \tau, \\
    1 - e^{-10(t-\tau)/T}, & t>\tau
\end{cases} \quad
m_t^{\text{linear}} = \begin{cases}
    0, &  t \le \tau, \\
    \frac{t-\tau}{T-\tau}, & t>\tau
\end{cases}
\end{equation}
where $t\in[T]$ and $\tau$ is uniformly chosen from $\{1,2,\dots, \lfloor\beta T\rfloor\}$ with $\beta=\frac{2}{3}$. These three patterns have variations bounded by $1$.

To generate opponents' highest bid sequences with larger variations, we partition the time horizon $T$ into $\lceil V_T\rceil$ segments, each of length at least $3$, and apply one of the three building blocks from Equation \eqref{eq:blocks} to each segment. The fourth pattern is generated by the sinusoidal wave $m_t=\frac{1}{2}+\frac{1}{2}\sin\left(\frac{V_T\pi t}{T}\right)$, which is employed in \citet{besbes2019optimal,cheung2022hedging}.

\textbf{Experimental setup:} We choose $T\in\{5000,8000,\dots, 59000\}$ and let $V_T=\frac{1}{4}\cdot T^\alpha$ for $\alpha\in\{0.1,0.3,0.5,0.7,0.9\}$. The learner's values $(v_t)_{t=1}^T$ are drawn i.i.d. from the uniform distribution on $[0,1]$. We consider the following algorithms:
\begin{itemize}
    \item \textbf{Theoretical Bound}: The theoretical dynamic regret upper bound $O\left(\sqrt{TV_T}\ln T\right)$.
    \item \textbf{AR-Prod}: Algorithm \ref{alg:full_vt} for unknown $V_T$. We use $\eta=1$, $\epsilon=\frac{4}{\sqrt{T}}$, and $c=\frac{1}{T}$.
    \item \textbf{AR-OMD}: Algorithm \ref{alg:full_lt} for unknown $L_T$. We use $\epsilon=\frac{1}{T^{0.9}}$, $\eta=\sqrt{\ln(T^{0.9})}$, and we consider $m_t\ne m_{t-1}$ if $|m_t-m_{t-1}|\ge 10^{-6}$.
    \item \textbf{BOBW}: Algorithm \ref{alg:best} with best-of-both-worlds guarantee. We use the default parameters specified in Algorithm \ref{alg:best}.
\end{itemize}

\begin{figure}[htbp]
    \centering
    \caption{Slopes of dynamic regret rates against varying time horizons for different opponents' highest bid sequence patterns and values of $\alpha$.}
    \includegraphics[width=1\textwidth]{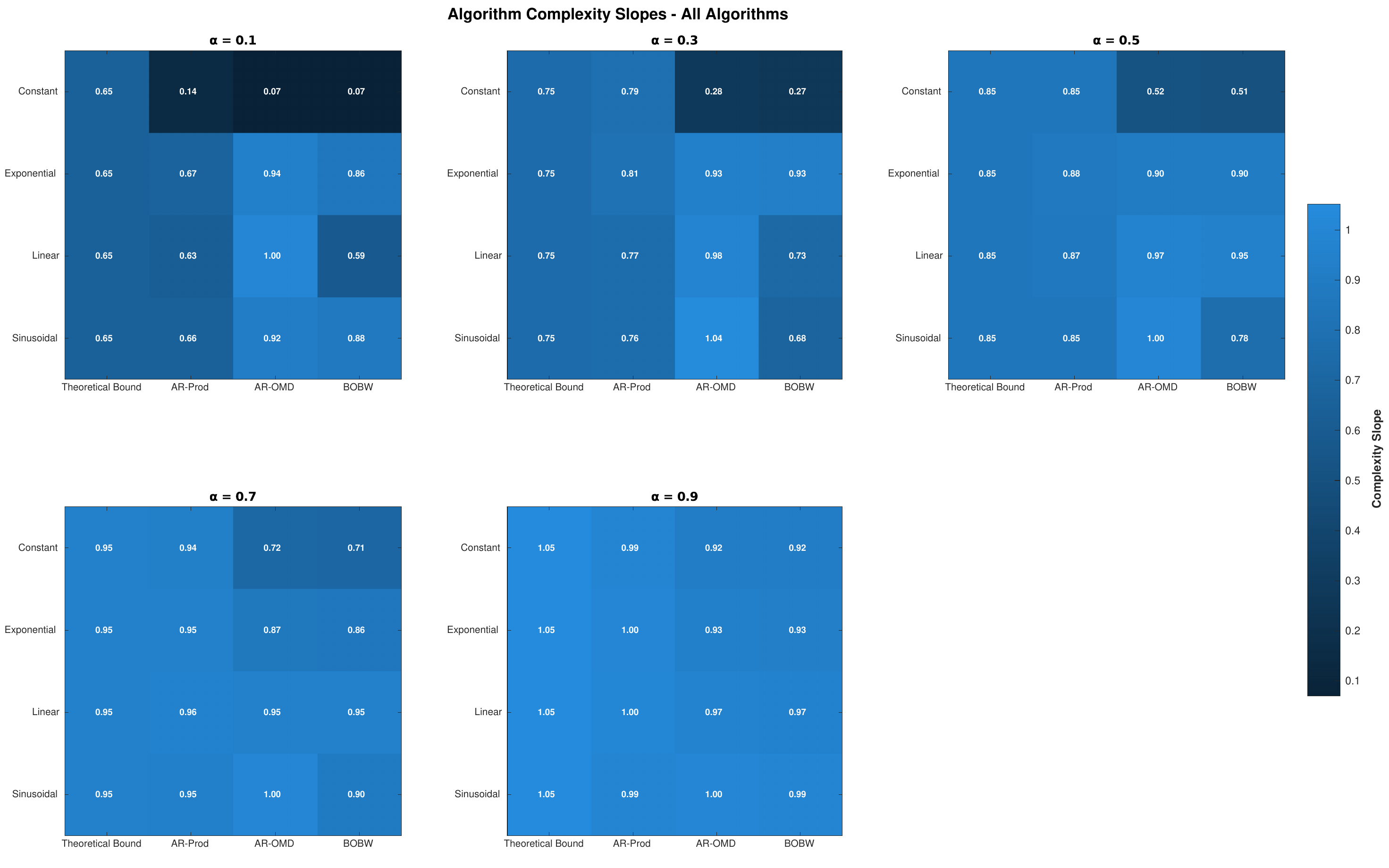}
    \label{fig:loglog_slope}
\end{figure}

\textbf{Results:} For each value of $\alpha$, we compute the slopes of log-log plots of average dynamic regret against $T$, and present all results in Figure \ref{fig:loglog_slope}. For all four temporal patterns, we observe that the slopes of AR-Prod align closely with the theoretical upper bound, which is consistent with Theorem \ref{thm:full_vt_ub}. For the \enquote{Constant} pattern, AR-OMD's slope approaches $\alpha$ because this pattern implies $V_T=L_T$ and its theoretical dynamic regret bound is $O\left(L_T\sqrt{\ln T}\right)=O\left(T^\alpha\sqrt{\ln T}\right)$ by Theorem \ref{thm:full_lt1}. For the remaining three patterns, the slopes of AR-OMD are close to $1$. This is because for these three continuously evolving patterns, $L_T=\Omega(T)$ even when $V_T$ is very small. Finally, BOBW's slope is consistently the minimum of the AR-Prod and AR-OMD slopes, confirming Theorem \ref{thm:bobw}.

\subsection{Performance Against Budget-Pacing Bidders}
We evaluate our algorithms against opponents using the budget-pacing algorithm from \citet{gaitonde2022budget}. We consider this experimental setting for two reasons: (i) the budget-pacing policy is an important strategy in both second-price auctions \citep{balseiro2019learning} and first-price auctions \citep{gaitonde2022budget}; (ii) when the value of each opponent varies slowly, the sequence of opponents' highest bids is also slowly varying. We find that our algorithms outperform the baselines, especially in regimes where opponents have limited budgets, even when the sequence of opponents' values varies rapidly.

For the budget-pacing algorithm from \citet{gaitonde2022budget}, the pacing multiplier $\mu_t$ is adaptively adjusted based on the observed expenditure to maintain a target spending rate $\rho_k = B_k/T$, where $k$ is the index of the $k$-th agent and $B_k$ is the initial budget of the $k$-th agent. We set $\epsilon_k=\frac{1}{\sqrt{T}}$ and $\bar{\mu}=\frac{T}{B_k}-1$ for this experiment based on suggestions from \citet{gaitonde2022budget}.

\textbf{Slowly Varying Property}: When opponent bidders have slowly varying valuations, the resulting sequence of opponents' highest bids is also slowly varying. This occurs because, given the current budget, Lagrangian multiplier, and ad impression value, each budget-pacing bidder's bid is precisely determined. This makes budget-pacing a suitable opponent strategy for validating our algorithms' performance in practical scenarios.

\textbf{Experimental setup}: We consider two budget regimes with 20 opponents:
\begin{itemize}
    \item \textbf{Sufficient budget}: Each opponent has budget $T/20$
    \item \textbf{Insufficient budget}: Each opponent has budget $T/40$
\end{itemize}
The sufficient budget regime corresponds to the case where the combined budget of all opponents is sufficient to purchase every ad impression, potentially leaving nothing for the learner. In contrast, the insufficient budget regime refers to the case where the opponents' budgets are collectively insufficient to purchase every ad impression.

The learner's values are i.i.d. uniform on $[0,1]$, while opponent values follow the four patterns from Section \ref{sec:loglog}, scaled by $0.8$ to ensure a reasonable winning probability for the learner. We set $T=12000$, $V_T=\frac{1}{4}\cdot T^\alpha$ with $\alpha\in\{0,0.1,\dots,1\}$, and average results over 50 runs. Baselines include the Hedge algorithm and the SEW algorithm \citep{han2020learning} \footnote{We use the official implementation of the SEW policy \citep{han_personal} in our numerical simulations.}.

\textbf{Results:} Figures \ref{fig:vary_alpha1} and \ref{fig:vary_alpha2} show that our algorithms, particularly AR-Prod, outperform baselines in both budget regimes, even in the case of $\alpha=1$. The advantage is more pronounced in the insufficient budget regime, where opponents bid more conservatively due to lower target spend rates, creating slowly varying opponents' highest bid sequences that our algorithms can exploit.

\begin{figure}
    \centering
    \caption{Cumulative rewards against budget-pacing bidders in the sufficient budget regime.}
    \includegraphics[width=1\textwidth]{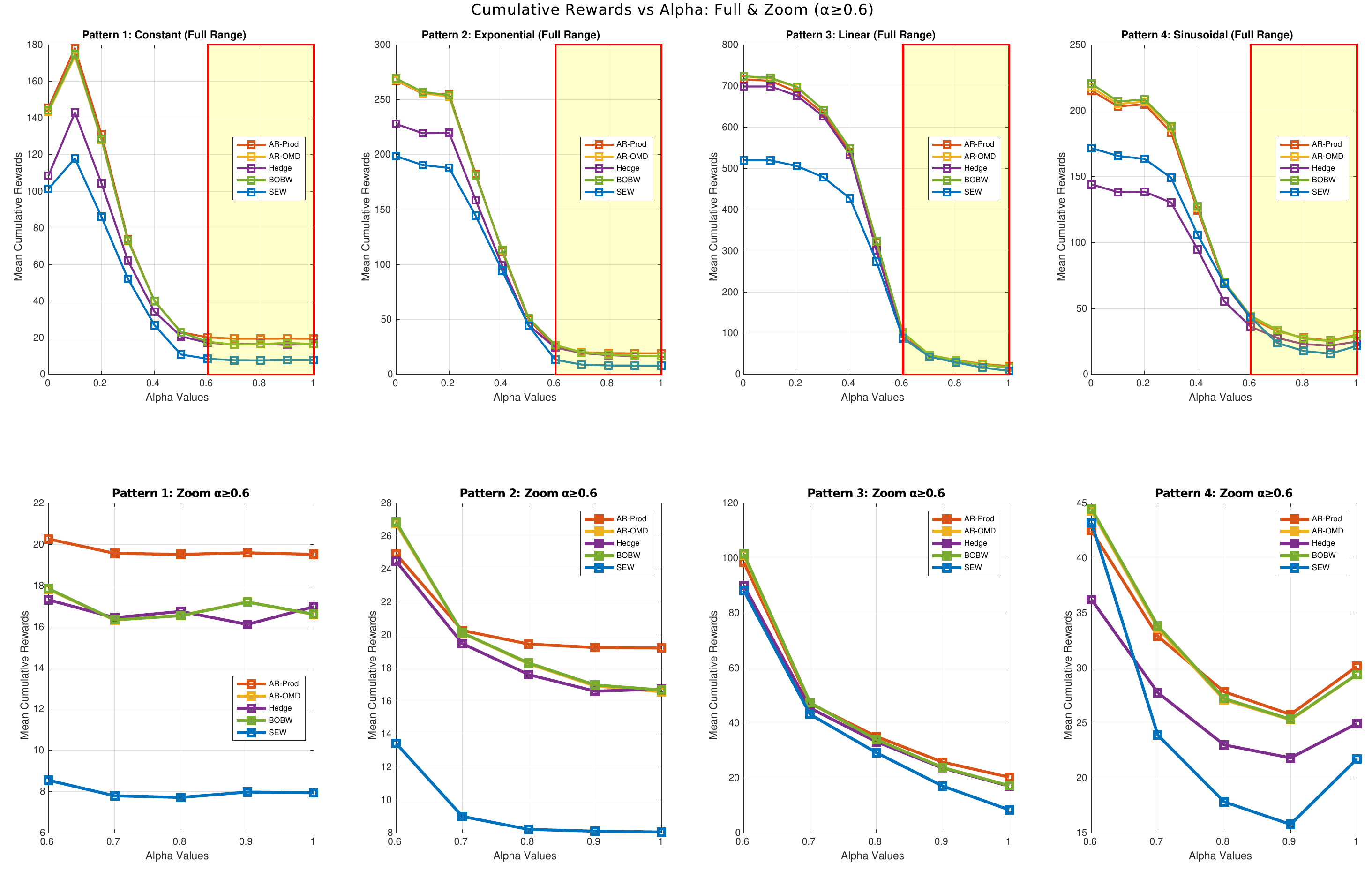}
    \label{fig:vary_alpha1}
\end{figure}

\begin{figure}
    \centering
    \caption{Cumulative rewards against budget-pacing bidders in the insufficient budget regime.}
    \includegraphics[width=1\textwidth]{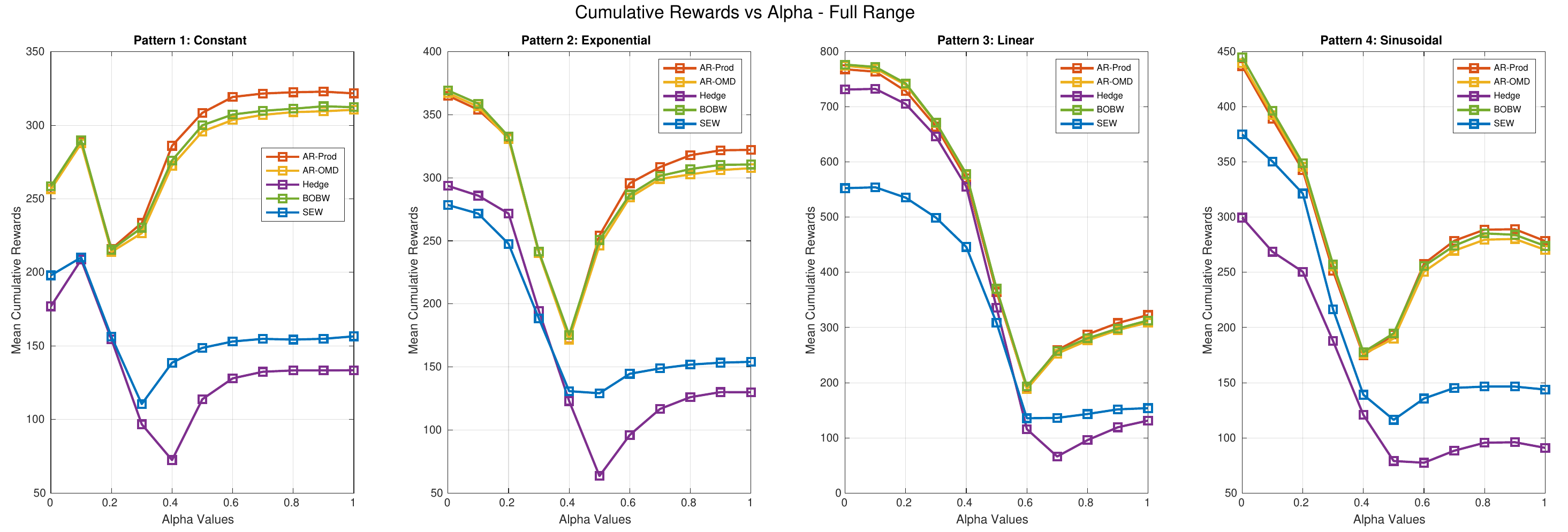}
    \label{fig:vary_alpha2}
\end{figure}

\textbf{Robustness analysis:} To ensure that our advantage arises from algorithmic adaptivity rather than simply from exploiting slowly varying sequences, we conducted additional experiments where both the learner's and the opponents' values are drawn from the same distributions: uniform, a truncated Gaussian (with mean $0.4$ and standard deviation $0.2$), or Beta$(3,3)$. In these experiments, we set $T=12000$ with $20$ budget-pacing bidders and vary the initial budget of each bidder. We repeat each experiment $50$ times and report the average performance. As shown in Figure \ref{fig:vary_b}, our algorithms remain competitive in regimes with sufficient budgets and outperform baselines in regimes with insufficient budgets, confirming that their adaptive capabilities extend beyond merely capitalizing on slow variation.
\begin{figure}
    \centering
    \caption{Cumulative rewards against budget-pacing bidders for different initial budgets.}
    \includegraphics[width=1\textwidth]{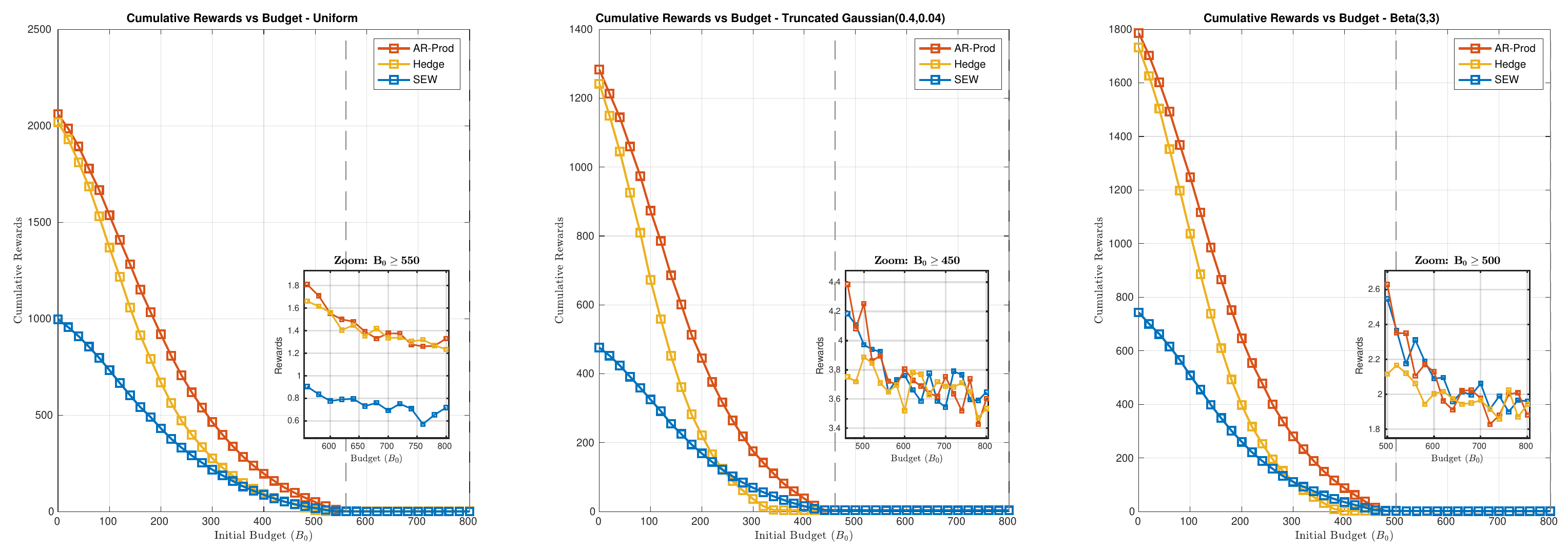}
    \label{fig:vary_b}
\end{figure}



\section{Conclusion}
\label{sec:conclusion}
This work examines online first-price auctions within non-stationary environments. While prior research typically focuses on competing against the best fixed policy in hindsight, such a policy can be suboptimal even in environments with mild non-stationarity. We instead investigate conditions under which competition against a dynamic benchmark, achieving the highest possible revenue, is feasible. We identify two measures of regularity on the opponents' highest bid sequence and establish minimax-optimal dynamic regret rates for the class of auction instances where the sequence of opponents' highest bids satisfies either of these regularity constraints.
For future work, it would be valuable to investigate tight dynamic regret rates under settings where only winning-bid feedback or binary feedback is available. From a technical perspective, our analysis considers the dynamic regret of a specific one-sided Lipschitz function with a single discontinuity. Given the existence of important one-sided Lipschitz functions with multiple discontinuities \citep{dutting2023optimal}, investigating the applicability of our algorithms to these more general settings presents a compelling research direction.

\ACKNOWLEDGMENT{
We thank Yanjun Han from New York University for providing the code of the SEW policy. This research is generously supported by the NSF grant CCF-2106508.
}



\bibliography{auction}

\begin{thebibliography}{71}
\providecommand{\natexlab}[1]{#1}
\providecommand{\url}[1]{\texttt{#1}}
\providecommand{\urlprefix}{URL }

\bibitem[{Aggarwal et~al.(2025)Aggarwal, Fikioris, \protect\BIBand{} Zhao}]{aggarwal2025no}
Aggarwal G, Fikioris G, Zhao M (2025) No-regret algorithms in non-truthful auctions with budget and roi constraints. \emph{Proceedings of the ACM on Web Conference 2025}, 1398--1415.

\bibitem[{Ai et~al.(2022)Ai, Wang, Li, Zhang, Huang, \protect\BIBand{} Deng}]{ai2022no}
Ai R, Wang C, Li C, Zhang J, Huang W, Deng X (2022) No-regret learning in repeated first-price auctions with budget constraints. \emph{arXiv preprint arXiv:2205.14572} .

\bibitem[{Akbarpour \protect\BIBand{} Li(2020)}]{akbarpour2020credible}
Akbarpour M, Li S (2020) Credible auctions: A trilemma. \emph{Econometrica} 88(2):425--467, \urlprefix\url{http://dx.doi.org/https://doi.org/10.3982/ECTA15925}.

\bibitem[{Alcobendas \protect\BIBand{} Zeithammer(2021)}]{alcobendas2021adjustment}
Alcobendas M, Zeithammer R (2021) Adjustment of bidding strategies after a switch to first-price rules. \emph{Available at SSRN 4036006} .

\bibitem[{Auer et~al.(2002)Auer, Cesa-Bianchi, \protect\BIBand{} Gentile}]{auer2002adaptive}
Auer P, Cesa-Bianchi N, Gentile C (2002) Adaptive and self-confident on-line learning algorithms. \emph{Journal of Computer and System Sciences} 64(1):48--75.

\bibitem[{Baby \protect\BIBand{} Wang(2021)}]{baby2021optimal}
Baby D, Wang YX (2021) Optimal dynamic regret in exp-concave online learning. \emph{Conference on Learning Theory}, 359--409 (PMLR).

\bibitem[{Baby \protect\BIBand{} Wang(2022)}]{baby2022optimal}
Baby D, Wang YX (2022) Optimal dynamic regret in proper online learning with strongly convex losses and beyond. \emph{International Conference on Artificial Intelligence and Statistics}, 1805--1845 (PMLR).

\bibitem[{Badanidiyuru et~al.(2023)Badanidiyuru, Feng, \protect\BIBand{} Guruganesh}]{badanidiyuru2023learning}
Badanidiyuru A, Feng Z, Guruganesh G (2023) Learning to bid in contextual first price auctions. \emph{Proceedings of the ACM Web Conference 2023}, 3489--3497.

\bibitem[{Balseiro et~al.(2023)Balseiro, Golrezaei, Mahdian, Mirrokni, \protect\BIBand{} Schneider}]{balseiro2023contextual}
Balseiro S, Golrezaei N, Mahdian M, Mirrokni V, Schneider J (2023) Contextual bandits with cross-learning. \emph{Mathematics of Operations Research} 48(3):1607--1629.

\bibitem[{Balseiro \protect\BIBand{} Gur(2019)}]{balseiro2019learning}
Balseiro SR, Gur Y (2019) Learning in repeated auctions with budgets: Regret minimization and equilibrium. \emph{Management Science} 65(9):3952--3968.

\bibitem[{Banchio \protect\BIBand{} Mantegazza(2023)}]{banchio2023adaptive}
Banchio M, Mantegazza G (2023) Adaptive algorithms and collusion via coupling. \emph{EC}, 208.

\bibitem[{Banchio \protect\BIBand{} Skrzypacz(2022)}]{banchio2022artificial}
Banchio M, Skrzypacz A (2022) Artificial intelligence and auction design. \emph{Proceedings of the 23rd ACM Conference on Economics and Computation}, 30--31.

\bibitem[{Besbes et~al.(2015)Besbes, Gur, \protect\BIBand{} Zeevi}]{besbes2015non}
Besbes O, Gur Y, Zeevi A (2015) Non-stationary stochastic optimization. \emph{Operations research} 63(5):1227--1244.

\bibitem[{Besbes et~al.(2019)Besbes, Gur, \protect\BIBand{} Zeevi}]{besbes2019optimal}
Besbes O, Gur Y, Zeevi A (2019) Optimal exploration-exploitation in a multi-armed-bandit problem with non-stationary rewards. \emph{Stochastic Systems} 9(4):319--337.

\bibitem[{Bichler et~al.(2023)Bichler, Fichtl, \protect\BIBand{} Oberlechner}]{bichler2023computing}
Bichler M, Fichtl M, Oberlechner M (2023) Computing bayes--nash equilibrium strategies in auction games via simultaneous online dual averaging. \emph{Operations Research} .

\bibitem[{Bigler(2019)}]{bigler2019rolling}
Bigler J (2019) Rolling out first price auctions to {Google Ad Manager} partners. \urlprefix\url{https://blog.google/products/admanager/rolling-out-first-price-auctions-google-ad-manager-partners/}, director, Product Management, Google.

\bibitem[{Castiglioni et~al.(2022)Castiglioni, Celli, \protect\BIBand{} Kroer}]{castiglioni2022online}
Castiglioni M, Celli A, Kroer C (2022) Online learning with knapsacks: the best of both worlds. \emph{International Conference on Machine Learning}, 2767--2783 (PMLR).

\bibitem[{Cesa-Bianchi et~al.(2024)Cesa-Bianchi, Cesari, Colomboni, Fusco, \protect\BIBand{} Leonardi}]{cesa2024role}
Cesa-Bianchi N, Cesari T, Colomboni R, Fusco F, Leonardi S (2024) The role of transparency in repeated first-price auctions with unknown valuations. \emph{Proceedings of the 56th Annual ACM Symposium on Theory of Computing}, 225--236.

\bibitem[{Cesa-Bianchi et~al.(2017)Cesa-Bianchi, Gaillard, Gentile, \protect\BIBand{} Gerchinovitz}]{cesa2017algorithmic}
Cesa-Bianchi N, Gaillard P, Gentile C, Gerchinovitz S (2017) Algorithmic chaining and the role of partial feedback in online nonparametric learning. \emph{Conference on Learning Theory}, 465--481 (PMLR).

\bibitem[{Cesa-Bianchi \protect\BIBand{} Lugosi(2006)}]{cesa2006prediction}
Cesa-Bianchi N, Lugosi G (2006) \emph{Prediction, learning, and games} (Cambridge university press).

\bibitem[{Cesa-Bianchi et~al.(2007)Cesa-Bianchi, Mansour, \protect\BIBand{} Stoltz}]{cesa2007improved}
Cesa-Bianchi N, Mansour Y, Stoltz G (2007) Improved second-order bounds for prediction with expert advice. \emph{Machine Learning} 66:321--352.

\bibitem[{Chen et~al.(2025)Chen, Wang, \protect\BIBand{} Wang}]{chen2025learning}
Chen N, Wang C, Wang L (2025) Learning and optimization with seasonal patterns. \emph{Operations Research} 73(2):894--909, \urlprefix\url{http://dx.doi.org/10.1287/opre.2023.0017}.

\bibitem[{Chen \protect\BIBand{} Peng(2023)}]{chen2023complexity}
Chen X, Peng B (2023) Complexity of equilibria in first-price auctions under general tie-breaking rules. \emph{Proceedings of the 55th Annual ACM Symposium on Theory of Computing}, 698--709.

\bibitem[{Cheung et~al.(2022)Cheung, Simchi-Levi, \protect\BIBand{} Zhu}]{cheung2022hedging}
Cheung WC, Simchi-Levi D, Zhu R (2022) Hedging the drift: Learning to optimize under nonstationarity. \emph{Management Science} 68(3):1696--1713.

\bibitem[{Cheung et~al.(2023)Cheung, Simchi-Levi, \protect\BIBand{} Zhu}]{cheung2023nonstationary}
Cheung WC, Simchi-Levi D, Zhu R (2023) Nonstationary reinforcement learning: The blessing of (more) optimism. \emph{Management Science} 69(10):5722--5739.

\bibitem[{Chiang et~al.(2012)Chiang, Yang, Lee, Mahdavi, Lu, Jin, \protect\BIBand{} Zhu}]{chiang2012online}
Chiang CK, Yang T, Lee CJ, Mahdavi M, Lu CJ, Jin R, Zhu S (2012) Online optimization with gradual variations. \emph{Conference on Learning Theory}, 6--1 (JMLR Workshop and Conference Proceedings).

\bibitem[{Choi et~al.(2020)Choi, Mela, Balseiro, \protect\BIBand{} Leary}]{choi2020online}
Choi H, Mela CF, Balseiro SR, Leary A (2020) Online display advertising markets: A literature review and future directions. \emph{Information Systems Research} 31(2):556--575.

\bibitem[{Conitzer et~al.(2022)Conitzer, Kroer, Panigrahi, Schrijvers, Stier-Moses, Sodomka, \protect\BIBand{} Wilkens}]{conitzer2022pacing}
Conitzer V, Kroer C, Panigrahi D, Schrijvers O, Stier-Moses NE, Sodomka E, Wilkens CA (2022) Pacing equilibrium in first price auction markets. \emph{Management Science} 68(12):8515--8535.

\bibitem[{Cover(1966)}]{cover1966behavior}
Cover TM (1966) \emph{Behavior of sequential predictors of binary sequences}. Number 7002 (Stanford University, Stanford Electronics Laboratories, Systems Theory~…).

\bibitem[{Despotakis et~al.(2021)Despotakis, Ravi, \protect\BIBand{} Sayedi}]{despotakis2021first}
Despotakis S, Ravi R, Sayedi A (2021) First-price auctions in online display advertising. \emph{Journal of Marketing Research} 58(5):888--907.

\bibitem[{D{\"u}tting et~al.(2023)D{\"u}tting, Guruganesh, Schneider, \protect\BIBand{} Wang}]{dutting2023optimal}
D{\"u}tting P, Guruganesh G, Schneider J, Wang JR (2023) Optimal no-regret learning for one-sided lipschitz functions. \emph{International Conference on Machine Learning}, 8836--8850 (PMLR).

\bibitem[{Edelman et~al.(2007)Edelman, Ostrovsky, \protect\BIBand{} Schwarz}]{edelman2007internet}
Edelman B, Ostrovsky M, Schwarz M (2007) Internet advertising and the generalized second-price auction: Selling billions of dollars worth of keywords. \emph{American economic review} 97(1):242--259.

\bibitem[{Fikioris \protect\BIBand{} Tardos(2023)}]{fikioris2023liquid}
Fikioris G, Tardos {\'E} (2023) Liquid welfare guarantees for no-regret learning in sequential budgeted auctions. \emph{Proceedings of the 24th ACM Conference on Economics and Computation}, 678--698.

\bibitem[{Filos-Ratsikas et~al.(2024)Filos-Ratsikas, Giannakopoulos, Hollender, \protect\BIBand{} Kokkalis}]{filos2024computation}
Filos-Ratsikas A, Giannakopoulos Y, Hollender A, Kokkalis C (2024) On the computation of equilibria in discrete first-price auctions. \emph{arXiv preprint arXiv:2402.12068} .

\bibitem[{Filos-Ratsikas et~al.(2021)Filos-Ratsikas, Giannakopoulos, Hollender, Lazos, \protect\BIBand{} Po{\c{c}}as}]{filos2021complexity}
Filos-Ratsikas A, Giannakopoulos Y, Hollender A, Lazos P, Po{\c{c}}as D (2021) On the complexity of equilibrium computation in first-price auctions. \emph{Proceedings of the 22nd ACM Conference on Economics and Computation}, 454--476.

\bibitem[{Gaitonde et~al.(2022)Gaitonde, Li, Light, Lucier, \protect\BIBand{} Slivkins}]{gaitonde2022budget}
Gaitonde J, Li Y, Light B, Lucier B, Slivkins A (2022) Budget pacing in repeated auctions: Regret and efficiency without convergence. \emph{arXiv preprint arXiv:2205.08674} .

\bibitem[{{Google Developers}(2024)}]{google_openrtb_adx_proto}
{Google Developers} (2024) {OpenRTB Extensions Protocol Buffer | Real-time Bidding | Google for Developers}. \url{https://developers.google.com/authorized-buyers/rtb/downloads/openrtb-adx-proto}, accessed: 2025-01-03.

\bibitem[{Gravin et~al.(2016)Gravin, Peres, \protect\BIBand{} Sivan}]{gravin2016towards}
Gravin N, Peres Y, Sivan B (2016) Towards optimal algorithms for prediction with expert advice. \emph{Proceedings of the twenty-seventh annual ACM-SIAM symposium on Discrete algorithms}, 528--547 (SIAM).

\bibitem[{Han(2024)}]{han_personal}
Han Y (2024) {SEW} algorithm implementation. Personal communication.

\bibitem[{Han et~al.(2025)Han, Weissman, \protect\BIBand{} Zhou}]{han2025optimal}
Han Y, Weissman T, Zhou Z (2025) Optimal no-regret learning in repeated first-price auctions. \emph{Operations Research} 73(1):209--238, \urlprefix\url{http://dx.doi.org/10.1287/opre.2020.0282}.

\bibitem[{Han et~al.(2020)Han, Zhou, Flores, Ordentlich, \protect\BIBand{} Weissman}]{han2020learning}
Han Y, Zhou Z, Flores A, Ordentlich E, Weissman T (2020) Learning to bid optimally and efficiently in adversarial first-price auctions. \emph{arXiv preprint arXiv:2007.04568} .

\bibitem[{Harvey et~al.(2023)Harvey, Liaw, Perkins, \protect\BIBand{} Randhawa}]{harvey2023optimal}
Harvey NJ, Liaw C, Perkins E, Randhawa S (2023) Optimal anytime regret with two experts. \emph{Mathematical Statistics and Learning} 6(1):87--142.

\bibitem[{Huang \protect\BIBand{} Wang(2025)}]{huang2025stability}
Huang C, Wang K (2025) A stability principle for learning under nonstationarity. \emph{Operations Research} \urlprefix\url{http://dx.doi.org/10.1287/opre.2024.0766}.

\bibitem[{Jadbabaie et~al.(2015)Jadbabaie, Rakhlin, Shahrampour, \protect\BIBand{} Sridharan}]{jadbabaie2015online}
Jadbabaie A, Rakhlin A, Shahrampour S, Sridharan K (2015) Online optimization: Competing with dynamic comparators. \emph{Artificial Intelligence and Statistics}, 398--406 (PMLR).

\bibitem[{Jin \protect\BIBand{} Lu(2023)}]{jin2023first}
Jin Y, Lu P (2023) First price auction is 1-1/e 2 efficient. \emph{Journal of the ACM} 70(5):1--86.

\bibitem[{Kumar et~al.(2024)Kumar, Schneider, \protect\BIBand{} Sivan}]{kumar2024strategically}
Kumar R, Schneider J, Sivan B (2024) Strategically-robust learning algorithms for bidding in first-price auctions. \emph{arXiv preprint arXiv:2402.07363} .

\bibitem[{Lucier et~al.(2024)Lucier, Pattathil, Slivkins, \protect\BIBand{} Zhang}]{lucier2024autobidders}
Lucier B, Pattathil S, Slivkins A, Zhang M (2024) Autobidders with budget and roi constraints: Efficiency, regret, and pacing dynamics. \emph{The Thirty Seventh Annual Conference on Learning Theory}, 3642--3643 (PMLR).

\bibitem[{Lucking-Reiley(2000)}]{lucking2000vickrey}
Lucking-Reiley D (2000) Vickrey auctions in practice: From nineteenth-century philately to twenty-first-century e-commerce. \emph{Journal of economic perspectives} 14(3):183--192.

\bibitem[{{Microsoft Learn Challenge}(2024)}]{xandr2024auction}
{Microsoft Learn Challenge} (2024) Auction overview. \urlprefix\url{https://learn.microsoft.com/en-us/xandr/bidders/auction-overview}, [Online; accessed 2024-02-07].

\bibitem[{Myerson(1981)}]{myerson1981optimal}
Myerson RB (1981) Optimal auction design. \emph{Mathematics of operations research} 6(1):58--73.

\bibitem[{Rakhlin \protect\BIBand{} Sridharan(2013)}]{rakhlin2013optimization}
Rakhlin S, Sridharan K (2013) Optimization, learning, and games with predictable sequences. \emph{Advances in Neural Information Processing Systems} 26.

\bibitem[{Rothkopf et~al.(1990)Rothkopf, Teisberg, \protect\BIBand{} Kahn}]{rothkopf1990vickrey}
Rothkopf MH, Teisberg TJ, Kahn EP (1990) Why are vickrey auctions rare? \emph{Journal of Political Economy} 98(1):94--109.

\bibitem[{Sani et~al.(2014)Sani, Neu, \protect\BIBand{} Lazaric}]{sani2014exploiting}
Sani A, Neu G, Lazaric A (2014) Exploiting easy data in online optimization. \emph{Advances in Neural Information Processing Systems} 27.

\bibitem[{Schneider \protect\BIBand{} Zimmert(2024)}]{schneider2024optimal}
Schneider J, Zimmert J (2024) Optimal cross-learning for contextual bandits with unknown context distributions. \emph{Advances in Neural Information Processing Systems} 36.

\bibitem[{Simchi-Levi et~al.(2023)Simchi-Levi, Wang, \protect\BIBand{} Zheng}]{simchi2023non}
Simchi-Levi D, Wang C, Zheng Z (2023) Non-stationary experimental design under linear trends. \emph{Advances in Neural Information Processing Systems} 36:32102--32116.

\bibitem[{{Statista}(2023)}]{statistaad}
{Statista} (2023) Digital advertising: market data \& analysis. \url{https://www.statista.com/study/42540/digital-advertising-report/}, released: December 2023.

\bibitem[{Steinhardt \protect\BIBand{} Liang(2014)}]{steinhardt2014adaptivity}
Steinhardt J, Liang P (2014) Adaptivity and optimism: An improved exponentiated gradient algorithm. \emph{International conference on machine learning}, 1593--1601 (PMLR).

\bibitem[{Syrgkanis et~al.(2015)Syrgkanis, Agarwal, Luo, \protect\BIBand{} Schapire}]{syrgkanis2015fast}
Syrgkanis V, Agarwal A, Luo H, Schapire RE (2015) Fast convergence of regularized learning in games. \emph{Advances in Neural Information Processing Systems} 28.

\bibitem[{Vickrey(1961)}]{vickrey1961counterspeculation}
Vickrey W (1961) Counterspeculation, auctions, and competitive sealed tenders. \emph{The Journal of finance} 16(1):8--37.

\bibitem[{Wang et~al.(2017)Wang, Zhang, Yuan et~al.}]{wang2017display}
Wang J, Zhang W, Yuan S, et~al. (2017) Display advertising with real-time bidding (rtb) and behavioural targeting. \emph{Foundations and Trends{\textregistered} in Information Retrieval} 11(4-5):297--435.

\bibitem[{Wang et~al.(2023)Wang, Yang, Deng, \protect\BIBand{} Kong}]{wang2023learning}
Wang Q, Yang Z, Deng X, Kong Y (2023) Learning to bid in repeated first-price auctions with budgets. \emph{International Conference on Machine Learning}, 36494--36513 (PMLR).

\bibitem[{Wang et~al.(2020)Wang, Shen, \protect\BIBand{} Zuo}]{wang2020bayesian}
Wang Z, Shen W, Zuo S (2020) Bayesian nash equilibrium in first-price auction with discrete value distributions. \emph{Proceedings of the 19th International Conference on Autonomous Agents and MultiAgent Systems}, 1458--1466.

\bibitem[{Wei \protect\BIBand{} Luo(2018)}]{wei2018more}
Wei CY, Luo H (2018) More adaptive algorithms for adversarial bandits. \emph{Conference On Learning Theory}, 1263--1291 (PMLR).

\bibitem[{Wei \protect\BIBand{} Luo(2021)}]{wei2021non}
Wei CY, Luo H (2021) Non-stationary reinforcement learning without prior knowledge: An optimal black-box approach. \emph{Conference on learning theory}, 4300--4354 (PMLR).

\bibitem[{Wong(2021)}]{wong2021moving}
Wong M (2021) Moving {AdSense} to a first-price auction. \urlprefix\url{https://blog.google/products/adsense/our-move-to-a-first-price-auction/}, product Manager, Google AdSense.

\bibitem[{Yang et~al.(2016)Yang, Zhang, Jin, \protect\BIBand{} Yi}]{yang2016tracking}
Yang T, Zhang L, Jin R, Yi J (2016) Tracking slowly moving clairvoyant: Optimal dynamic regret of online learning with true and noisy gradient. \emph{International Conference on Machine Learning}, 449--457 (PMLR).

\bibitem[{Zhang et~al.(2018)Zhang, Lu, \protect\BIBand{} Zhou}]{zhang2018adaptive}
Zhang L, Lu S, Zhou ZH (2018) Adaptive online learning in dynamic environments. \emph{Advances in neural information processing systems} 31.

\bibitem[{Zhang et~al.(2022)Zhang, Han, Zhou, Flores, \protect\BIBand{} Weissman}]{zhang2022leveraging}
Zhang W, Han Y, Zhou Z, Flores A, Weissman T (2022) Leveraging the hints: Adaptive bidding in repeated first-price auctions. \emph{Advances in Neural Information Processing Systems} 35:21329--21341.

\bibitem[{Zhang et~al.(2021)Zhang, Kitts, Han, Zhou, Mao, He, Pan, Flores, Gultekin, \protect\BIBand{} Weissman}]{zhang2021meow}
Zhang W, Kitts B, Han Y, Zhou Z, Mao T, He H, Pan S, Flores A, Gultekin S, Weissman T (2021) Meow: A space-efficient nonparametric bid shading algorithm. \emph{Proceedings of the 27th ACM SIGKDD Conference on Knowledge Discovery \& Data Mining}, 3928--3936.

\bibitem[{Zhao \protect\BIBand{} Chen(2020)}]{haoyu2020online}
Zhao H, Chen W (2020) Online second price auction with semi-bandit feedback under the non-stationary setting. \emph{Proceedings of the AAAI Conference on Artificial Intelligence}, volume~34, 6893--6900.

\bibitem[{Zhao et~al.(2021)Zhao, Zhang, Jiang, \protect\BIBand{} Zhou}]{zhao2021simple}
Zhao P, Zhang L, Jiang Y, Zhou ZH (2021) A simple approach for non-stationary linear bandits. \emph{arXiv preprint arXiv:2103.05324} .

\end{thebibliography}
\clearpage

\end{document}